\documentclass{article}
\usepackage{iclr2027_conference,times}
\usepackage{enumitem}

\usepackage{amsmath,amsfonts,amsthm,bm}
\usepackage{algorithm}
\usepackage{algorithmic}
\usepackage{array}
\usepackage{booktabs}
\usepackage{graphicx}

\usepackage{hyperref}
\usepackage{cleveref}

\usepackage{multirow}
\usepackage{subcaption}
\usepackage{tikz}
\usepackage{wrapfig}    
\usepackage{url}
\usepackage[dvipsnames]{xcolor}
\usepackage{xspace}

\makeatletter
\DeclareRobustCommand\onedot{\futurelet\@let@token\@onedot}
\def\@onedot{\ifx\@let@token.\else.\null\fi\xspace}

\makeatother

\crefname{appendixsec}{Appendix}{Appendices}
\Crefname{appendixsec}{Appendix}{Appendices}

\newcommand{\ignore}[1]{}   

\newcolumntype{L}[1]{>{\raggedright\let\newline\\\arraybackslash\hspace{0pt}}m{#1}}
\newcolumntype{C}[1]{>{\centering\let\newline\\\arraybackslash\hspace{0pt}}m{#1}}
\newcolumntype{R}[1]{>{\raggedleft\let\newline\\\arraybackslash\hspace{0pt}}m{#1}}

\newcommand{\modelname}{SWIM\xspace}

\definecolor{visualhlt}{HTML}{FAE7CB}  
\definecolor{audiohlt}{HTML}{EEE2F7}   

\definecolor{tabfirst}{rgb}{1, 0.7, 0.7} 
\definecolor{tabsecond}{rgb}{1, 0.85, 0.7} 
\definecolor{tabthird}{rgb}{1, 1, 0.7} 

\newtheorem{theorem}{Theorem}
\newtheorem{remark}{Remark}
\newtheorem{proposition}{Proposition}

\def\eqref#1{equation~\ref{#1}}

\def\1{\bm{1}}

\DeclareMathAlphabet{\mathsfit}{\encodingdefault}{\sfdefault}{m}{sl}
\SetMathAlphabet{\mathsfit}{bold}{\encodingdefault}{\sfdefault}{bx}{n}

\title{Spend Where It Matters:\\ Particle Search with MH Refinement for Budget-Constrained Audio-Visual Alignment}

\author{
Shivanshu Shekhar$^{1,2}$\thanks{Work done during an internship at Adobe Research.},
Raghavendra Addanki$^{2}$,
Tong Zhang$^{1}$,
Uttaran Bhattacharya$^{2}$\\
$^{1}$University of Illinois Urbana-Champaign \qquad
$^{2}$Adobe Research
}

\iclrfinalcopy 

\begin{document}

\maketitle

\lhead{Preprint}

\begin{abstract}
Reward-guided inference-time alignment can improve joint audio-visual generation, but optimizing multiple alignment objectives requires many expensive model evaluations, making existing methods costly under limited compute. Particle-based search enables global exploration but incurs estimation error with a finite population, while Metropolis-Hastings (MH) can refine candidates locally but benefits from a strong initialization. We introduce \textbf{S}pend \textbf{W}here \textbf{I}t \textbf{M}atters (\modelname), a two-stage, training-free, reward-agnostic framework that combines particle search with MH refinement under a fixed compute budget. In the first stage, we adaptively allocate particles across denoising checkpoints using online variance estimates, minimizing estimation variance while concentrating compute where it is most useful. This search provides a promising candidate for MH refinement, which uses local noise-and-denoise proposals on completed outputs to further improve reward. By starting from a strong candidate, MH can recover additional reward without requiring a larger particle population or costly global exploration. We show that, at matched compute, combining particle search with MH refinement achieves higher rewards than either approach alone. On \textsc{JavisBench-Mini}, at closely matched compute, \modelname achieves mean relative gains across seven metrics of $30.7\%$ over BoN and $14.1\%$ over EvoSearch with JavisDiT++, and $9.2\%$ over BoN and $17.0\%$ over EvoSearch with LTX-2. Beyond joint audio-visual generation, we also demonstrate the benefits of adaptive particle allocation on SD1.5 image generation, improving AestheticScore by $2.2\%$ and ImageReward by $6.0\%$ over uniform allocation.
\end{abstract}

\section{Introduction}\label{sec:intro}

Diffusion and flow-based models have achieved high-fidelity generation across images~\citep{ldm, lfm}, video~\citep{svd, wan}, text~\citep{llada, mdm}, and audio~\citep{aldm, fastaldm}. As individual modalities mature, attention has shifted toward \emph{joint multi-modal generation}. Joint audio-visual (AV) video generation is particularly compelling because audio and visual signals are closely coupled in the real world, with direct applications in media such as film, gaming, and AR/VR. While recent models such as JavisDiT++~\citep{javisdit} and LTX-2~\citep{ltx2} jointly generate visuals and audio from text within a single flow-matching framework, they still struggle to align the generation with the text prompts, as well as align the generated audio and visual modalities with each other.

Such alignment issues are well-studied in text and image generation, and two main approaches have emerged. \emph{Post-training alignment} uses RL to align generative models with rewards~\citep{ddpo, seedpo, rocm, draft}, and \emph{inference-time alignment} steers generation toward higher-reward outputs without modifying model weights~\citep{fkd, vasr, dts}. In this work, we adopt inference-time alignment for joint AV generation because it is training-free and modular, allowing us to adapt to different alignment objectives simply by swapping reward models without retraining the generator.

Among inference-time alignment methods, \emph{particle-based methods} have shown promise as they offer efficient, parallelizable global exploration without requiring reward gradients~\citep{fkd, fkc, vasr}. However, limited particle populations can lead to high estimation variance and degraded sample quality. Metropolis-Hastings (MH)~\citep{dmh} enables local refinement but requires substantial compute for global exploration due to slow mixing~\citep{MH_mixing}. More recently, ITS-AVGen~\citep{itsavgen} explores joint AV alignment using multiple verifiers and adaptive reward weighting to optimize coarse and fine-grained objectives simultaneously. However, fine-grained reward estimates can be unreliable before the coarse AV structure is established, potentially degrading alignment. 

To address these challenges, we introduce \textbf{S}pend \textbf{W}here \textbf{I}t \textbf{M}atters (\modelname{}), a two-stage inference-time alignment framework that combines budget-constrained particle search with MH refinement for coarse-to-fine alignment. Motivated by VASR~\citep{vasr}, we formulate particle allocation as a variance-minimization problem under a fixed compute budget and derive a closed-form optimal schedule. Using online sampling-variance estimates, we adapt the particle population across denoising checkpoints to establish coarse alignment and obtain a promising candidate under a fixed budget. MH then refines this candidate through local noise-and-denoise proposals, optimizing fine-grained alignment after the coarse structure is established. This division of labor reduces the need for costly global exploration during MH while recovering reward without requiring a larger particle population, enabling efficient multi-objective alignment under a fixed compute budget.

Our contributions are:
\begin{itemize}[
    leftmargin=1em,
    labelsep=0.4em,
    topsep=2pt,
    itemsep=1pt,
    parsep=0pt
]


\item \textbf{Particle-initialized MH for budget-constrained alignment.} We introduce a two-stage strategy that first explores several denoising trajectories with particle search, then uses the particle with the highest particle-stage reward to initialize a short MH chain. For the same total number of model evaluations, this strategy achieves a higher mean rewards than particle search alone or MH from an unguided sample at every tested budget.

\item \textbf{Online variance-based particle allocation.} We formulate particle allocation across denoising checkpoints as a variance-minimization problem under a fixed budget and derive a closed-form solution. During each generation, we use online variance estimates to adapt the particle population at each checkpoint to the prompt and the corresponding reward signals.

\item \textbf{Evaluation on joint audio-visual generation.} On \textsc{JavisBench-Mini}, at closely matched compute, \modelname achieves mean relative gains across seven metrics of 30.7\% over BoN and 14.1\% over EvoSearch with JavisDiT++, and 9.2\% over BoN and 17.0\% over EvoSearch with LTX-2. In image generation, on SD1.5 at 400 forward passes, it improves AestheticScore by $2.2\%$ and ImageReward by $6.0\%$ over uniform allocation.
\end{itemize}

\section{Related Work}\label{sec:related}

\paragraph{Audio-visual alignment.}
Recent work has explored both training-based and inference-time approaches to improving alignment in joint audio-visual video generation. SyncDPO~\citep{syncdpo} applies direct preference optimization using temporally misaligned negatives, while SpongeBob~\citep{spongebob} introduces a dual-stream Diffusion Transformer with bidirectional cross-modal attention and synchronization-aware training. These methods improve alignment by modifying the underlying generative model, whereas our method operates entirely at inference time and can be applied to pretrained backbones without additional training. More recently, ITS-AVGen~\citep{itsavgen} studies inference-time scaling for joint AV generation using multiple verifiers and introduces Adaptive Reward Weighting (ARW) to calibrate heterogeneous reward signals at test time. In contrast, our method performs multi-stage inference-time alignment, first achieving coarse alignment and then applying a finer adjustment stage for further refinement, without additional training.

\paragraph{Efficient inference-time alignment.}
Our approach builds on inference-time alignment and particle-based sampling for diffusion models. Best-of-$N$ sampling generates independent candidates but cannot reuse computation from unsuccessful trajectories. Feynman-Kac (FK) steering~\citep{fkd} instead uses interacting particles and intermediate resampling guided by reward-derived potentials, while FK Correctors~\citep{fkc} extend this framework to multi-objective generation through product-of-experts distributions. VASR~\citep{vasr} reduces SMC variance through variance-optimal particle resampling, whereas DTS*~\citep{dts} uses a monte-carlo tree search to find high reward sample. Particle-based search offers global exploration and can reuse computation, but finite particle populations introduce estimation error when approximating the reward-twisted distribution. Building on VASR, we further reduce estimation variance by adaptively allocating particles across denoising checkpoints under a fixed compute budget. Our allocation strategy is compatible with different resampling rules. We refine the candidate using MH, which achieves greater reward improvements than allocating the same compute to additional particles.

\section{Preliminaries}\label{sec:prelim}

\paragraph{Diffusion and flow-matching models.}
Denoising diffusion probabilistic models~\citep{ddpm} use the forward process $q(x_t \mid x_0) = \mathcal{N}(x_t;\,\sqrt{\bar\alpha_t}\,x_0,\,(1{-}\bar\alpha_t)I)$ and train a noise predictor $\epsilon_\theta(x_t,t)$. In a continuous-time score formulation, the reverse-time SDE and probability-flow ODE share marginals but have different drifts~\citep{smld}. Flow matching~\citep{fm,rfm} trains a velocity field $v_\theta(x_t,t) \approx \varepsilon - x_0$ on the interpolant $x_t = (1{-}t)x_0 + t\varepsilon$, $\varepsilon \sim \mathcal{N}(0,I_d)$, and integrating backward from noise gives a noise-to-data flow. We write our sampler as
\begin{equation}\label{eqn:sampling-dynamics}
    dX_t = \underbrace{\tilde{b}_\theta(X_t, t)}_{\text{drift}}\,dt + \underbrace{g(t)}_{\text{diffusion}}\,dW_t,
\end{equation}
where drift and diffusion depend on the sampler. Our flow-model sampler uses $g(t)=\eta\,t$. The denoised estimate $\hat{x}_0(x_t,t)$ approximates $\mathbb{E}[x_0\mid x_t]$ and can be decoded for reward evaluation. We use the sampler associated with each underlying model.


\paragraph{Sequential Monte Carlo (SMC).}
For a reference generator $p_{\mathrm{ref}}(x\mid p)$ and reward $R(x,p)$, the ideal reward-tilted target is $\pi(x\mid p)\propto p_{\mathrm{ref}}(x\mid p)\exp(\lambda R(x,p))$. SMC uses weighted particles and resampling to approximate such target distributions~\citep{smc}. In our search, intermediate rewards guide particle selection, and the terminal best-particle choice is a search decision rather than a draw from $\pi$.

At checkpoint $m$, let $w_m^{(k)}$ denote the nonnegative selection weight of arriving particle $k$. The effective sample size
$\mathrm{ESS}_m=(\sum_k w_m^{(k)})^2/\sum_k(w_m^{(k)})^2$
measures weight concentration. We use the empirical squared coefficient of variation,
\begin{equation}\label{eqn:smc-variance}
    \hat\sigma_m^2=\frac{K_{m-1}}{\mathrm{ESS}_m}-1,
\end{equation}
to guide compute across checkpoints. The sum $\sum_m\hat\sigma_m^2/K_{m-1}$ is a tractable variance \emph{surrogate}, not the exact variance of the final output.

\paragraph{Metropolis-Hastings (MH).}
For a target distribution $\pi$, MH~\citep{mh,mh2} proposes $x'\sim q(\cdot\mid x)$ and accepts it with probability
\begin{equation}\label{eqn:mh-accept-prelim}
    \alpha(x,x') = \min\!\left\{1,\;\frac{\pi(x')\,q(x \mid x')}{\pi(x)\,q(x' \mid x)}\right\}.
\end{equation}
Our proposal adds noise to the output latent and then re-denoises it along a new stochastic path. If the reverse process is exact, the acceptance rule for the reward-tilted target reduces to
$\alpha=\min\{1,\exp(\lambda\Delta R)\}$ (\Cref{app:mh-acceptance}).
With a learned reverse model, this reward-only acceptance rule is approximate. Under standard conditions, repeated exact MH steps converge to the target distribution, while a finite run also depends on its initialization and proposal scale.

\section{Method}\label{sec:method}

\modelname combines adaptive particle search with Metropolis-Hastings (MH) refinement under a fixed compute budget. Stage 1 uses an adaptive particle funnel (\Cref{subsec:funnel}) to identify a promising candidate under budget $B_{\mathrm{part}}$, while Stage 2 applies $T$ MH proposals (\Cref{subsec:mh}) to further improve its reward. If proposal $j$ costs $c_j$ generator forward evaluations, the total budget is
$
B_{\mathrm{total}}=B_{\mathrm{part}}+\sum_{j=1}^{T}c_j.
$
The budget split between the two stages is fixed, while our online allocation strategy adapts computation within the particle stage.

\subsection{Adaptive Particle Funnel}\label{subsec:funnel}

\paragraph{Formulation.}
Consider a diffusion or flow-matching model with $N$ denoising steps and $M$
fixed checkpoints at noise levels $t_1 > \cdots > t_M$. Let $K_0$ denote the
initial particle population, and let $K_m$ denote the population after
resampling at checkpoint $m$. Let $n_1$ denote the number of steps before the
first checkpoint. For $m<M$, let $L_m$ denote the segment length between
checkpoints $m$ and $m{+}1$. Let $L_M$ denote the final segment after
checkpoint $M$. The cost to propagate $K_m$ particles for $L_m$ steps is
$K_mL_m$ generator forward evaluations. At checkpoint \(m\), we score the \(K_{m-1}\) arriving particles and compute
the selection weights $ w_k =
\exp\!\left(\lambda\left(R_k-R_{\max}\right)\right).
$

We measure the dispersion of these weights by their squared coefficient of
variation. Its empirical estimate is
$
\widehat{\sigma}_m^2
=
\frac{K_{m-1}}{\operatorname{ESS}_m}-1.
$
This estimate is zero when all particles have equal weight. It increases as
the weights become concentrated on fewer particles. We use this dispersion to
define the estimation-variance surrogate
$
V_{\mathrm{sur}}
=
\sum_{m=1}^{M}
\frac{\sigma_m^2}{K_{m-1}}.
$

Here, \(\sigma_m^2\) denotes the anticipated weight dispersion at checkpoint
\(m\), while \(\widehat{\sigma}_m^2\) denotes its observed estimate. We treat
the dispersion coefficients as fixed. We then choose the population sizes to
minimize this surrogate under the particle budget.

We set $K_M=1$ and
define: 
$\tilde B = B_{\mathrm{part}}-L_M.$
The continuous allocation problem is
\begin{equation}
\label{eqn:opt-problem}
\min_{K_1,\ldots,K_{M-1}}
\;\sum_{m=1}^{M-1}\frac{\sigma_{m+1}^2}{K_m}
\qquad
\text{s.t.}\quad
\sum_{m=1}^{M-1}K_mL_m
=
\tilde B-K_0n_1.
\end{equation}
We treat $K_0$ as a hyperparameter. It sets the initial exploration population
before online dispersion estimates are available. We deduct its cost,
$K_0n_1$, from the budget. We allocate the remaining budget over
$K_1,\ldots,K_{M-1}$. Let $
B_{\mathrm{seg}}=\tilde B-K_0n_1$ denote the budget available for $K_1,\ldots,K_{M-1}$.

\begin{theorem}[Optimal allocation for the variance surrogate, proof in \Cref{app:discrete-alloc}]\label{thm:opt-alloc}
The solution to \Cref{eqn:opt-problem} is
\begin{equation}\label{eqn:Kstar}
    K_m^*
    =\frac{\tilde B-K_0n_1}{\mathcal{Z}}
    \cdot\frac{\sigma_{m+1}}{\sqrt{L_m}},
    \qquad 1\le m\le M{-}1,
\end{equation}
where $\mathcal{Z}=\sum_{m=1}^{M-1}\sigma_{m+1}\sqrt{L_m}$.
\end{theorem}

The optimal allocation assigns more particles to stages with greater anticipated weight dispersion and fewer to longer, more expensive stages.

\paragraph{Online algorithm.}

Theorem~\ref{thm:opt-alloc}  gives the optimal allocation when the future dispersions
$\{\sigma_{m+1}\}$ are known. Its optimality condition can equivalently be written as: 
$
K_m^*
=
{\sigma_{m+1}}/{\sqrt{\nu L_m}},
$
where the multiplier $\nu$ sets the overall compute scale so that the allocations satisfy the budget.

At inference time, neither the future dispersion nor the corresponding multiplier is known.
We use the current dispersion as a one-step forecast,
$\hat{\sigma}_{m+1|m}^2=\hat{\sigma}_m^2$.
We then adapt the multiplier from the compute used so far.
Before choosing $K_m$, let
$
C_{m-1}
=
\sum_{j=1}^{m-1}K_jL_j
$
be the allocation compute already spent, and let
$
T_{m-1}
=
B_{\mathrm{seg}}
\frac{\sum_{j=1}^{m-1}L_j}
{\sum_{j=1}^{M-1}L_j}
$
be the corresponding proportional-spending target.
Starting from $\nu_1=\nu_{\min}$, we set
\begin{equation}
\nu_m
=
\max\left\{
\nu_{\min},
\nu_{m-1}
+
\eta_\nu
\frac{C_{m-1}-T_{m-1}}{B_{\mathrm{seg}}}
\right\},
\end{equation}
and choose
$
K_m
\approx
\sqrt{
{\hat{\sigma}_{m+1|m}^2}
/{\nu_m L_m}
}.
$
If previous segments used more compute than the proportional target,
$\nu_m$ increases, and the next population shrinks. If they used less,
$\nu_m$ decreases, and the next population can grow. When the population grows, the growth rule determines which particles are duplicated (\Cref{fig:grow-mode}).

The step size $\eta_\nu$ controls how quickly the multiplier responds to budget error.
Because the one-step dispersion forecast can be inaccurate, we clip the proposed population between a pacing floor $K_{\mathrm{pace}}$, which preserves a fraction $\gamma$ of the uniform allocation, and an affordability cap $K_m^{\max}$, which reserves enough compute for the remaining segments (\Cref{app:minimax-pace,app:budget}).
\begin{algorithm}[t]
\caption{\modelname: Particle Funnel and MH Refinement}
\label{alg:funnel}
\begin{algorithmic}[1]
\REQUIRE Prompt $p$, particle budget $B_{\mathrm{part}}$, MH steps $T$, checkpoints $\{n_1,\ldots,n_M\}$, rewards $R_t,R_{\mathrm{MH}}$
\STATE Initialize $K_0$ particles from i.i.d.\ noise, set $\nu\leftarrow\nu_{\min}$, $B_{\mathrm{rem}}\leftarrow B_{\mathrm{part}}$
\FOR{denoising step $i=0,\ldots,N-1$}
    \IF{$i=n_m$ for some checkpoint $m$}
        \STATE Score all particles, $\hat{\sigma}_m^2\leftarrow K_{\mathrm{cur}}/\mathrm{ESS}-1$
        \STATE $\hat{\sigma}_{m+1|m}^2\leftarrow\hat{\sigma}_m^2$
        \STATE Update $\nu$, compute $K_m^{\max}$ and $K_{\mathrm{pace}}$
        \STATE $K_m\leftarrow\mathrm{clip}\!\left(
        \mathrm{round}\left(
        \sqrt{\hat{\sigma}_{m+1|m}^2/(\nu L_m)}
        \right),
        K_{\mathrm{pace}},K_m^{\max}
        \right)$
        \STATE Resample to $K_m$ particles
    \ENDIF
    \STATE Advance all particles by one denoising step
    \STATE $B_{\mathrm{rem}}\leftarrow B_{\mathrm{rem}}-K_{\mathrm{cur}}$
\ENDFOR
\STATE $x^*\leftarrow\arg\max_k R(x_0^{(k)})$
\FOR{MH step $j=1,\ldots,T$}
    \STATE Propose $x'\leftarrow\textsc{NoiseDenoise}(x^*,s_{\mathrm{mh}})$
    \STATE Accept $x^*\leftarrow x'$ with probability
    $\min\{1,\exp(\lambda_{\mathrm{MH}}\Delta R_{\mathrm{MH}})\}$
\ENDFOR
\RETURN $x^*$
\end{algorithmic}
\end{algorithm}
\vspace{-10pt}

\subsection{Metropolis-Hastings Refinement}\label{subsec:mh}
The particle funnel returns the terminal particle with the highest particle-stage reward, which initializes a short MH chain. At each MH step, we re-noise its output latent to level $s_{\mathrm{mh}}$ and re-denoise both modalities along a fresh stochastic path to obtain a proposal $x'$. Thus, $s_{\mathrm{mh}}$ sets the proposal scale and cost, while $T$ sets the number of refinement steps. We score each completed proposal with the fixed reward
$
R_{\mathrm{MH}}(x)
=
\sum_{q=1}^{Q}\beta_q R_q(x),
$
where $R_q$ is an alignment or quality objective and $\beta_q$ sets its weight. A proposal is accepted with probability
$
\alpha(x,x')
=
\min\left\{
1,
\exp\left(
\lambda_{\mathrm{MH}}
[R_{\mathrm{MH}}(x')-R_{\mathrm{MH}}(x)]
\right)
\right\}.
$
Under an exact reverse model, this reward-only rule is the MH acceptance ratio for the reward-tilted reference distribution. With a learned reverse model, it is approximate (\Cref{app:mh-acceptance}). Because we run only a short MH chain, its starting point matters. We therefore use particle search for global exploration and MH to refine its best terminal candidate
(\Cref{app:mh-convergence}).

\paragraph{Particle-stage reward.} Particle search uses a checkpoint-dependent reward:
\[
R_t(x)
=
R_1(x)
+
\beta_{\mathrm{mix}}
\mathbf{1}[t<t_{\mathrm{mix}}]R_2(x).
\]
Here, $R_t$ is the reward used to score particles at denoising time $t$. The term $R_1$ is used at every checkpoint, while $R_2$ is added once $t<t_{\mathrm{mix}}$, when its intermediate estimate is more reliable.

\section{Experiments}\label{sec:experiments}
\begin{figure}[!htbp]
    \centering
    \begin{minipage}{0.5\linewidth}\centering
        \includegraphics[width=\linewidth]{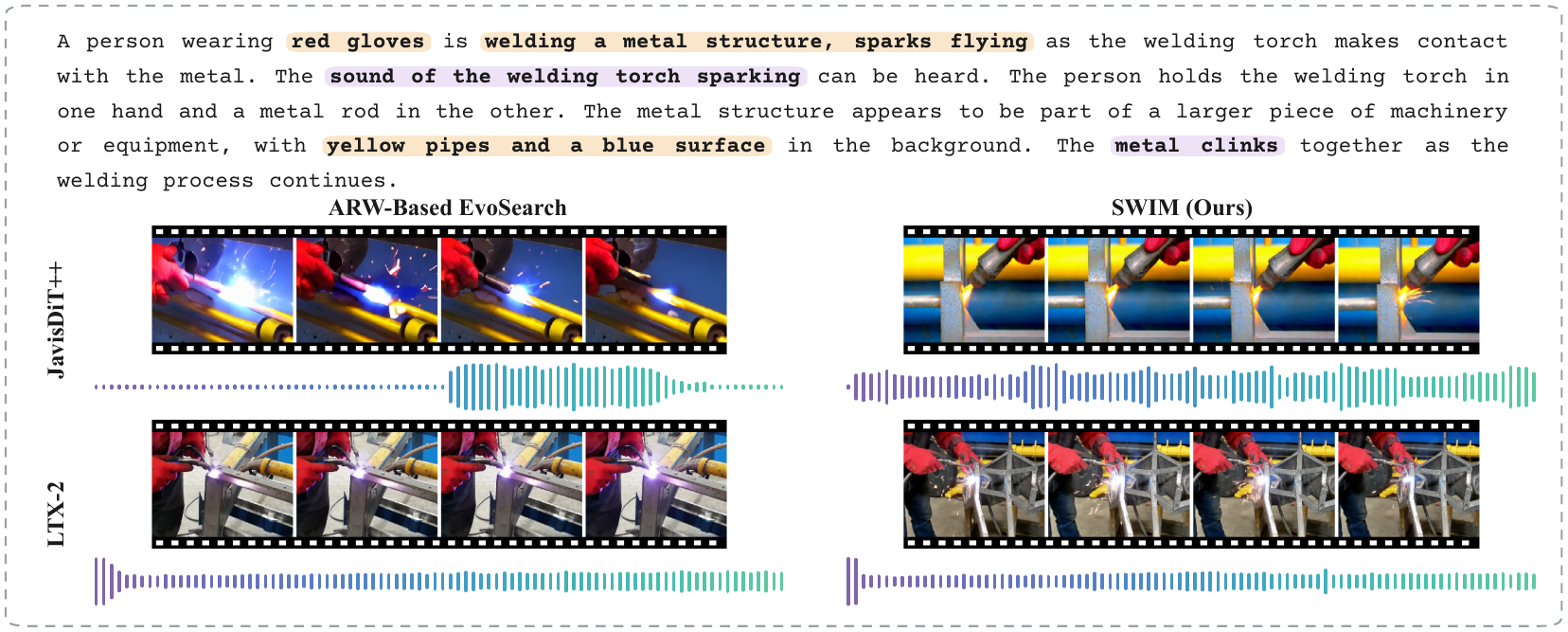}
    \end{minipage}%
    \begin{minipage}{0.5\linewidth}\centering
        \includegraphics[width=\linewidth]{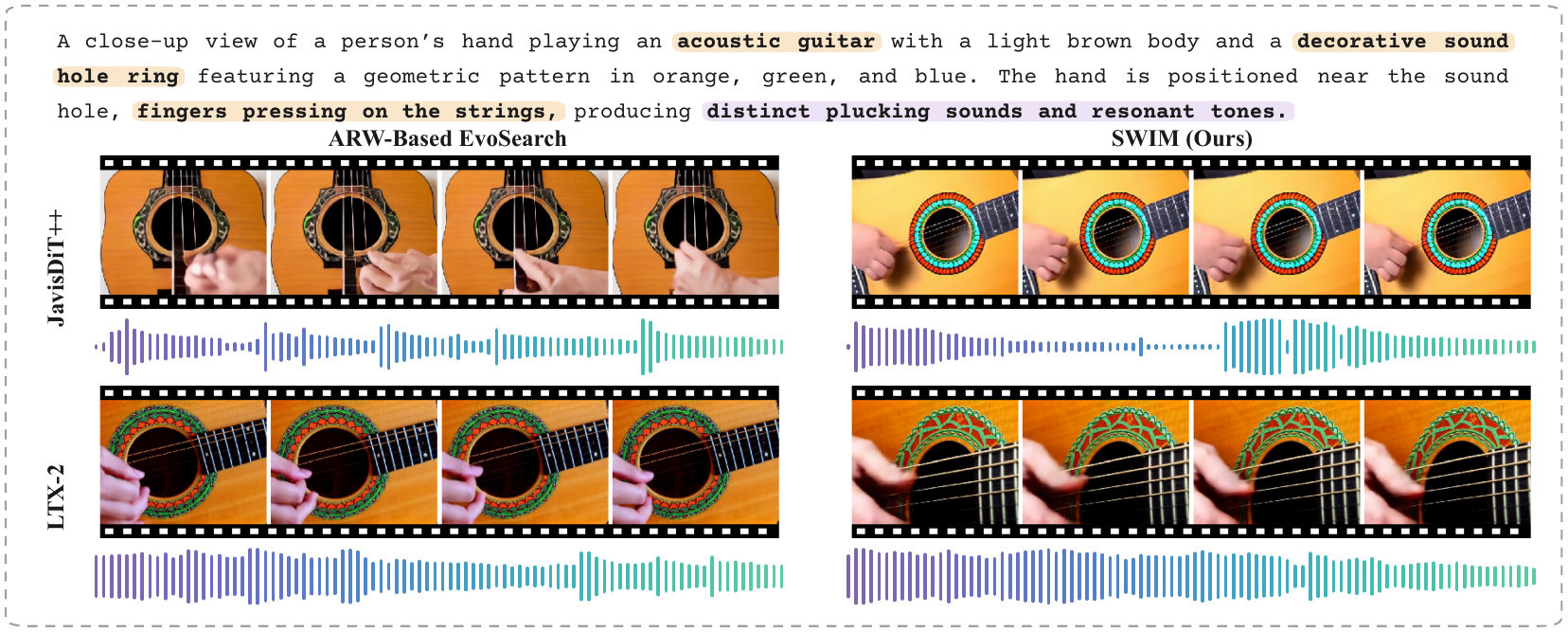}
    \end{minipage}

    \begin{minipage}{0.5\linewidth}\centering
        \includegraphics[width=\linewidth]{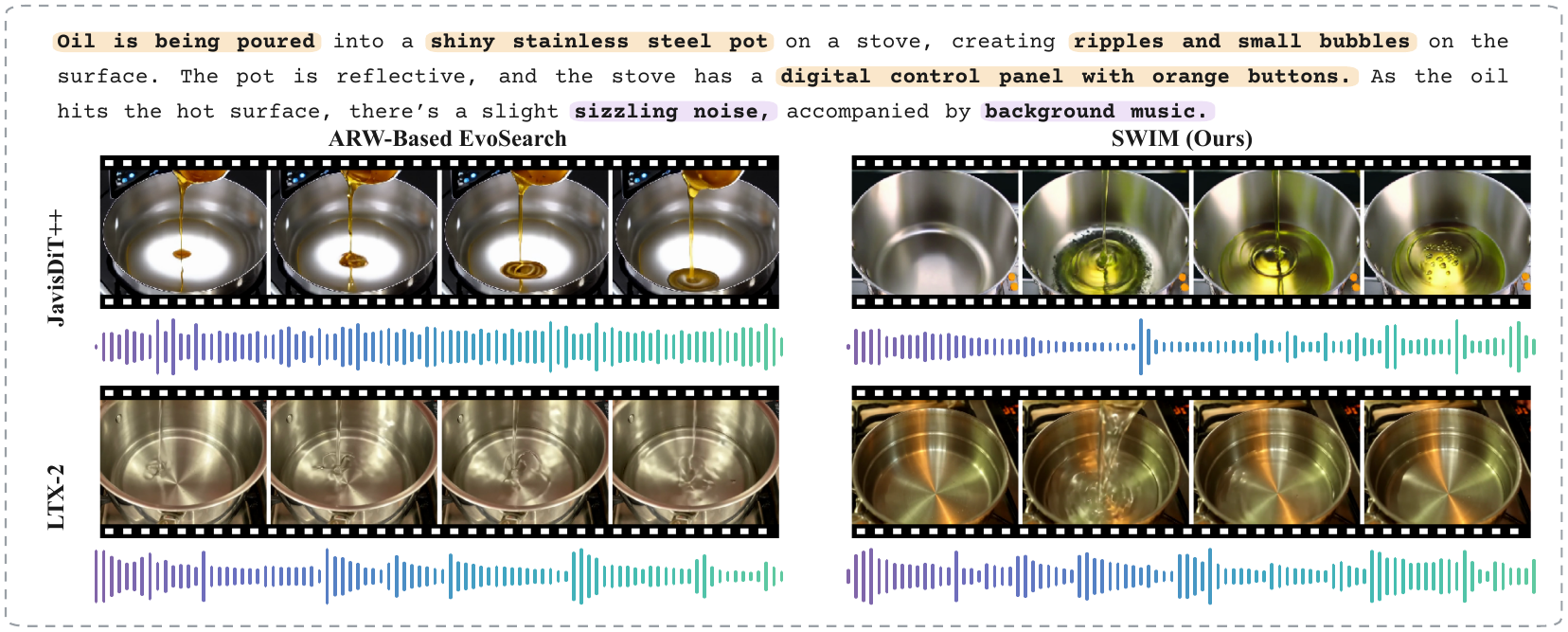}
    \end{minipage}%
    \begin{minipage}{0.5\linewidth}\centering
        \includegraphics[width=\linewidth]{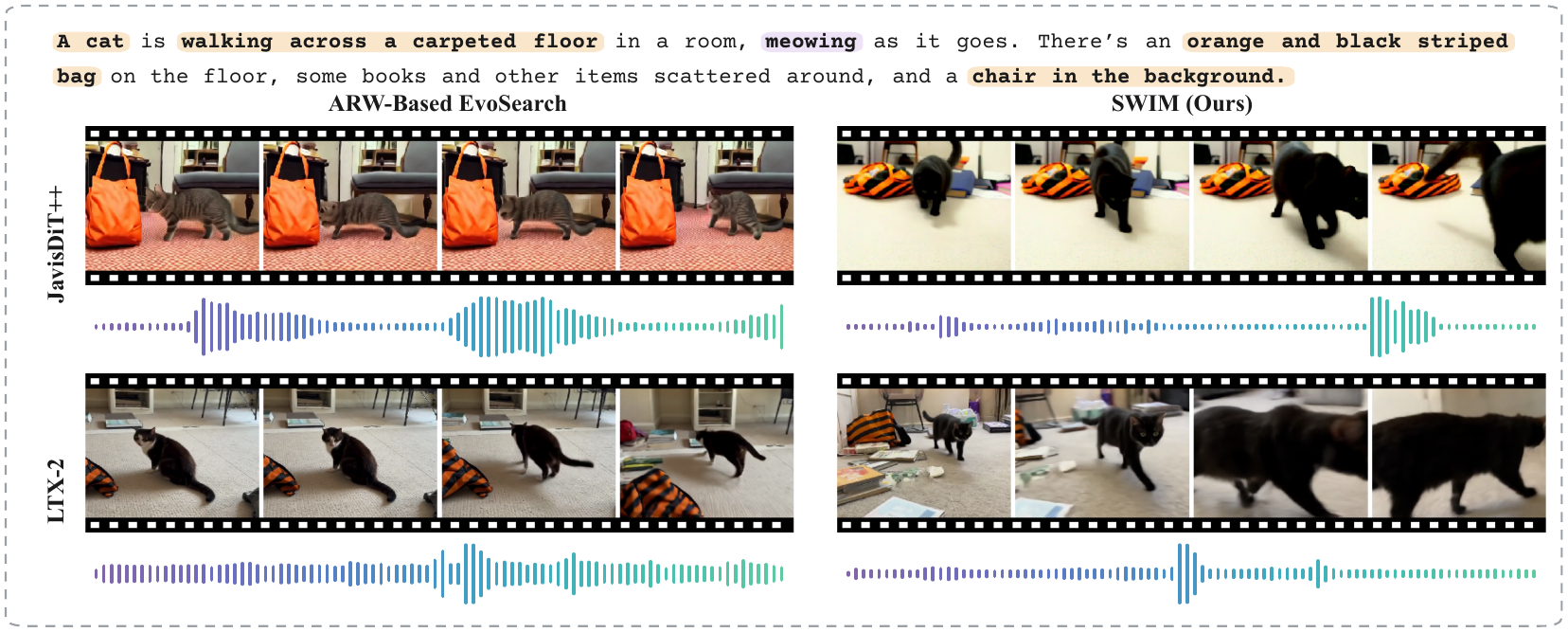}
    \end{minipage}
    \caption{Visual results on \textsc{JavisBench-Mini}~\citep{javis}. For each prompt, we show the ARW-based EvoSearch and \modelname outputs for both JavisDiT++ and LTX-2, each as a frame strip with its audio waveform below. Cues in the prompts are highlighted as {\setlength{\fboxsep}{1.5pt}\colorbox{visualhlt}{\textbf{\texttt{visual}}}} or
    {\setlength{\fboxsep}{1.5pt}\colorbox{audiohlt}{\textbf{\texttt{audio}}}}.}
    \label{fig:visual_results}
    \vspace{-5pt}
\end{figure}

We report our experimental setup, including models and benchmarks, baselines, metrics, and implementation, our main quantitative and qualitative results, and our ablation studies.

\subsection{Setup}

\paragraph{Models and benchmarks.}
We evaluate \modelname on two joint audio-visual models, JavisDiT++~\citep{javisdit} and LTX-2~\citep{ltx2}, on \textsc{JavisBench-Mini}~\citep{javis}. Beyond video generation, we also evaluate \modelname in the image domain with Stable Diffusion v1.5~\citep{ldm} on DrawBench~\citep{drawbench} using ImageReward~\citep{imagereward} and on Simple Animals~\citep{ddpo} using AestheticScore~\citep{aes}, following the image-domain evaluation setup of VASR~\citep{vasr}.

\paragraph{Baselines.}
We compare \modelname with Best-of-$N$ (BoN) and EvoSearch using the Adaptive Reward Weighting (ARW) scheme from ITS-AVGen~\citep{itsavgen}, and with Uniform+MH. We also include Uniform with a fixed particle population and no MH refinement, and view BoN as MH initialized from full noise. BoN generates independent samples and returns the highest-scoring one. EvoSearch~\citep{evo} performs evolutionary search along the denoising trajectory. Uniform+MH uses a uniform particle population at each checkpoint followed by MH refinement, isolating the effect of adaptive particle allocation. We report the NFE for each setting and compare methods at equal or closely matched budgets.

\paragraph{Metrics.}
For AV generation, we report text-visual ImageBind (TV-IB), text-audio ImageBind (TA-IB), audio-visual ImageBind (AV-IB)~\citep{ib}, AVHScore (AVH)~\citep{avh}, JavisScore (JS)~\citep{javis}, CLAP text-audio (CLAP)~\citep{clap}, and VideoReward text-alignment (VR-TA)~\citep{vr}. For image generation, we report ImageReward~\citep{imagereward} and AestheticScore~\citep{aes}. Higher scores are better for all metrics.

\paragraph{Implementation.}
Unless otherwise specified, all experiments use $N{=}50$ denoising steps and checkpoints at steps $\{10,20,30,40\}$. We set the particle-alignment temperature to $\lambda{=}1.0$. For joint audio-visual generation, the particle funnel uses the staged reward $R_t=R_{\mathrm{tv}}+\beta_{\mathrm{mix}}\mathbf{1}[t<t_{\mathrm{mix}}]R_{\mathrm{ta}}$, with $\beta_{\mathrm{mix}}{=}0.3$ and $t_{\mathrm{mix}}{=}0.85$. MH uses the fixed reward $R_{\mathrm{MH}}=\mathrm{JS}+\beta_{\mathrm{MH}}\mathrm{VR\text{-}TA}$ with $\beta_{\mathrm{MH}}{=}0.05$, which improves alignment while discouraging loss of video quality. Motivated by the image-domain results in \Cref{tab:drawbench_animals}, we use systematic resampling for AV-generation. We set the pacing floor to $\gamma{=}0.9$. For joint audio-visual generation, we evaluate two total compute budgets: $360$ NFE ($200$ NFE for particle search and $160$ NFE for MH) and $560$ NFE ($400$ NFE for particle search and $160$ NFE for MH), with initial populations of $K_0{=}6$ and $K_0{=}8$, respectively. MH refinement uses $T{=}20$ proposals, each costing $8$ NFE, with $\lambda_{\mathrm{MH}}{=}150$ and $s_{\mathrm{mh}}{=}0.7$.

\subsection{Main Results}

\paragraph{Audio-visual generation.}
\Cref{tab:javisbench} compares methods at lower and higher compute budgets. We take the mean of relative gains across the seven metrics and use the stronger of Uniform+MH and \modelname for each comparison, which is \modelname in all four model-budget settings. At the lower, closely matched budget, \modelname improves over BoN and EvoSearch by 30.7\% and 14.1\% with JavisDiT++, and by 9.2\% and 17.0\% with LTX-2. Against Uniform, its mean relative gains are 7.0\% with JavisDiT++ and 8.9\% with LTX-2. At the higher budget, its gains over BoN and EvoSearch are 31.5\% and 13.5\% with JavisDiT++, and 7.2\% and 10.9\% with LTX-2. Comparing \modelname with Uniform+MH isolates the effect of adaptive allocation. The mean relative gains are 1.0\% and 1.2\% at the lower budget, and 1.2\% and 3.0\% at the higher budget for JavisDiT++ and LTX-2, respectively.

We also show visual examples on four sample prompts from \textsc{JavisBench-Mini}~\citep{javis} in \Cref{fig:visual_results}, and include a larger set of 80 prompts in the supplementary material.

\paragraph{Image generation.}
\Cref{tab:drawbench_animals} shows that \modelname improves over the corresponding fixed-population Uniform+MH baseline under systematic, multinomial, and Fleming-Viot resampling, indicating that its gains do not depend on the resampling strategy. Taking the mean across AestheticScore and ImageReward, the relative gains are 4.1\% with systematic, 4.6\% with multinomial, and 3.3\% with Fleming-Viot resampling, for an overall mean of 4.0\%.

\begin{table}[!htbp]
    \centering
    \caption{Results on \textsc{JavisBench-Mini}~\citep{javis}. Budget is reported as total NFE (forward passes). \modelname achieves the strongest results at lower budgets, the practical regime for expensive AV models. Numbers reported as $\mathrm{mean} (\mathrm{std\ dev})$, $\uparrow:$ higher is better.}
    \label{tab:javisbench}
    \resizebox{\linewidth}{!}{
    \begin{tabular}{lllrrrrrr}
        \toprule
        Model & Method
        & TV-IB $\uparrow$
        & TA-IB $\uparrow$
        & AV-IB $\uparrow$
        & AVH $\uparrow$
        & JS $\uparrow$
        & CLAP $\uparrow$
        & VR-TA $\uparrow$ \\
        \midrule

        \multirow{8}{*}{JavisDiT++}
        & BoN@400 & $0.3780$ {\scriptsize$(0.0002)$} & $0.1496$ {\scriptsize$(0.0006)$} & $0.1859$ {\scriptsize$(0.0015)$} & $0.1711$ {\scriptsize$(0.0016)$} & $0.1021$ {\scriptsize$(0.0015)$} & $0.1474$ {\scriptsize$(0.0015)$} & $0.3711$ {\scriptsize$(0.0015)$} \\
        & BoN@600 & $0.3785$ {\scriptsize$(0.0003)$} & $0.1508$ {\scriptsize$(0.0013)$} & $0.1896$ {\scriptsize$(0.0018)$} & $0.1746$ {\scriptsize$(0.0013)$} & $0.1035$ {\scriptsize$(0.0003)$} & $0.1464$ {\scriptsize$(0.0017)$} & $0.3870$ {\scriptsize$(0.0001)$} \\
        & EvoSearch@399 & $0.3762$ {\scriptsize$(0.0011)$} & $0.1766$ {\scriptsize$(0.0012)$} & $0.2312$ {\scriptsize$(0.0005)$} & $0.2126$ {\scriptsize$(0.0016)$} & $0.1380$ {\scriptsize$(0.0016)$} & $0.1857$ {\scriptsize$(0.0014)$} & $0.3074$ {\scriptsize$(0.0048)$} \\
        & EvoSearch@532 & $0.3767$ {\scriptsize$(0.0006)$} & $0.1838$ {\scriptsize$(0.0022)$} & $0.2382$ {\scriptsize$(0.0003)$} & $0.2197$ {\scriptsize$(0.0013)$} & $0.1441$ {\scriptsize$(0.0004)$} & $\underline{0.1933}$ {\scriptsize$(0.0024)$} & $0.3223$ {\scriptsize$(0.0271)$} \\
        & Uniform@400 & $\mathbf{0.3824}$ {\scriptsize$(0.0004)$} & $0.1861$ {\scriptsize$(0.0026)$} & $0.2218$ {\scriptsize$(0.0016)$} & $0.2098$ {\scriptsize$(0.0019)$} & $0.1125$ {\scriptsize$(0.0023)$} & $0.1868$ {\scriptsize$(0.0066)$} & $\mathbf{0.6245}$ {\scriptsize$(0.0036)$} \\
        & Uniform+MH@360 & $0.3755$ {\scriptsize$(0.0010)$} & $\mathbf{0.1900}$ {\scriptsize$(0.0002)$} & $\underline{0.2434}$ {\scriptsize$(0.0021)$} & $\underline{0.2335}$ {\scriptsize$(0.0023)$} & $\underline{0.1692}$ {\scriptsize$(0.0024)$} & $\underline{0.1878}$ {\scriptsize$(0.0022)$} & $0.4417$ {\scriptsize$(0.0132)$} \\
        & Uniform+MH@560 & $\underline{0.3786}$ {\scriptsize$(0.0002)$} & $\underline{0.1885}$ {\scriptsize$(0.0042)$} & $\underline{0.2408}$ {\scriptsize$(0.0032)$} & $\underline{0.2304}$ {\scriptsize$(0.0033)$} & $\underline{0.1649}$ {\scriptsize$(0.0004)$} & $0.1894$ {\scriptsize$(0.0008)$} & $\mathbf{0.5325}$ {\scriptsize$(0.0132)$} \\
        & \textbf{\modelname@360} & $\underline{0.3795}$ {\scriptsize$(0.0009)$} & $\underline{0.1894}$ {\scriptsize$(0.0016)$} & $\mathbf{0.2442}$ {\scriptsize$(0.0018)$} & $\mathbf{0.2357}$ {\scriptsize$(0.0017)$} & $\mathbf{0.1695}$ {\scriptsize$(0.0013)$} & $\mathbf{0.1901}$ {\scriptsize$(0.0033)$} & $\underline{0.4587}$ {\scriptsize$(0.0040)$} \\
        & \textbf{\modelname@560} & $\mathbf{0.3789}$ {\scriptsize$(0.0004)$} & $\mathbf{0.1899}$ {\scriptsize$(0.0009)$} & $\mathbf{0.2448}$ {\scriptsize$(0.0020)$} & $\mathbf{0.2340}$ {\scriptsize$(0.019)$} & $\mathbf{0.1655}$ {\scriptsize$(0.0024)$} & $\mathbf{0.1973}$ {\scriptsize$(0.0012)$} & $\underline{0.5299}$ {\scriptsize$(0.0136)$} \\
        \midrule

        \multirow{8}{*}{LTX-2}
        & BoN@400 & $0.3797$ {\scriptsize$(0.0008)$} & $0.1598$ {\scriptsize$(0.0021)$} & $0.1984$ {\scriptsize$(0.0014)$} & $0.1823$ {\scriptsize$(0.0007)$} & $0.1055$ {\scriptsize$(0.007)$} & $0.1230$ {\scriptsize$(0.0035)$} & $0.4731$ {\scriptsize$(0.0020)$} \\
        & BoN@600 & $0.3802$ {\scriptsize$(0.0003)$} & $0.1625$ {\scriptsize$(0.0011)$} & $0.2061$ {\scriptsize$(0.0020)$} & $0.1905$ {\scriptsize$(0.0014)$} & $0.1042$ {\scriptsize$(0.0019)$} & $0.1394$ {\scriptsize$(0.0012)$} & $0.4983$ {\scriptsize$(0.0034)$} \\
        & EvoSearch@399 & $0.3775$ {\scriptsize$(0.0001)$} & $0.1633$ {\scriptsize$(0.0002)$} & $0.2044$ {\scriptsize$(0.0015)$} & $0.1394$ {\scriptsize$(0.0009)$} & $0.1014$ {\scriptsize$(0.0039)$} & $0.1421$ {\scriptsize$(0.0016)$} & $0.3571$ {\scriptsize$(0.0028)$} \\
        & EvoSearch@532 & $0.3778$ {\scriptsize$(0.0006)$} & $\underline{0.1688}$ {\scriptsize$(0.0011)$} & $\underline{0.2109}$ {\scriptsize$(0.0014)$} & $0.1949$ {\scriptsize$(0.0013)$} & $\underline{0.1095}$ {\scriptsize$(0.0022)$} & $\underline{0.1484}$ {\scriptsize$(0.0036)$} & $0.3643$ {\scriptsize$(0.0006)$} \\
        & Uniform@400 & $\mathbf{0.3823}$ {\scriptsize$(0.0001)$} & $0.1691$ {\scriptsize$(0.0024)$} & $0.2027$ {\scriptsize$(0.0006)$} & $0.1874$ {\scriptsize$(0.0005)$} & $0.0686$ {\scriptsize$(0.0018)$} & $0.1378$ {\scriptsize$(0.0034)$} & $\mathbf{0.6729}$ {\scriptsize$(0.0031)$} \\
        & Uniform+MH@360 & $0.3796$ {\scriptsize$(0.0002)$} & $\underline{0.1698}$ {\scriptsize$(0.0013)$} & $\underline{0.2106}$ {\scriptsize$(0.0046)$} & $\underline{0.1995}$ {\scriptsize$(0.0042)$} & $\underline{0.1081}$ {\scriptsize$(0.0025)$} & $\underline{0.1453}$ {\scriptsize$(0.0048)$} & $0.5329$ {\scriptsize$(0.0071)$} \\
        & Uniform+MH@560 & $\underline{0.3815}$ {\scriptsize$(0.0001)$} & $0.1672$ {\scriptsize$(0.0028)$} & $0.2099$ {\scriptsize$(0.0033)$} & $\underline{0.1974}$ {\scriptsize$(0.0028)$} & $0.1058$ {\scriptsize$(0.0027)$} & $0.1328$ {\scriptsize$(0.0029)$} & $\underline{0.6165}$ {\scriptsize$(0.0096)$} \\
        & \textbf{\modelname@360} & $\underline{0.3799}$ {\scriptsize$(0.0002)$} & $\mathbf{0.1701}$ {\scriptsize$(0.006)$} & $\mathbf{0.2111}$ {\scriptsize$(0.0007)$} & $\mathbf{0.2028}$ {\scriptsize$(0.0002)$} & $\mathbf{0.1105}$ {\scriptsize$(0.0016)$} & $\mathbf{0.1497}$ {\scriptsize$(0.0047)$} & $\underline{0.5387}$ {\scriptsize$(0.0106)$} \\
        & \textbf{\modelname@560} & $\mathbf{0.3822}$ {\scriptsize$(0.0006)$} & $\mathbf{0.1696}$ {\scriptsize$(0.0023)$} & $\mathbf{0.2114}$ {\scriptsize$(0.0060)$} & $\mathbf{0.2004}$ {\scriptsize$(0.0049)$} & $\mathbf{0.1102}$ {\scriptsize$(0.0017)$} & $\mathbf{0.1498}$ {\scriptsize$(0.0002)$} & $\mathbf{0.6193}$ {\scriptsize$(0.0003)$} \\
        \bottomrule
    \end{tabular}
    }
\end{table}

\begin{table}[!htbp]
    \centering
    \caption{Results on DrawBench (ImageReward)~\citep{drawbench} and Simple Animals (AestheticScore)~\citep{ddpo}. \modelname outperforms fixed-population baselines across all resampling strategies at matched NFE. Numbers reported as $\mathrm{mean} (\mathrm{std\ dev})$, $\uparrow:$ higher is better.}
    \label{tab:drawbench_animals}
    \resizebox{0.7\linewidth}{!}{
    \begin{tabular}{lcc}
        \toprule
        Method & SimpleAnimals: AestheticScore $\uparrow$ & DrawBench: ImageReward $\uparrow$ \\
        \midrule
        Uniform+MH(Sys)@400 & $6.506$ {\scriptsize$(0.018)$}  & $0.9890$ {\scriptsize$(0.027)$}  \\
        Uniform+MH(Mul)@400 & $6.395$ {\scriptsize$(0.016)$}  & $0.9129$ {\scriptsize$(0.030)$}  \\
        Uniform+MH(FV)@400 & $6.452$ {\scriptsize$(0.024)$}  & $0.9677$ {\scriptsize$(0.023)$}  \\
        \textbf{\modelname-Sys@400} & $\mathbf{6.652}$ {\scriptsize$(0.021)$}  & $\mathbf{1.048}$ {\scriptsize$(0.022)$}  \\
        \textbf{\modelname-Mul@400} & $\mathbf{6.528}$ {\scriptsize$(0.029)$}  & $\mathbf{0.978}$ {\scriptsize$(0.031)$}  \\
        \textbf{\modelname-FV@400}  & $\mathbf{6.602}$ {\scriptsize$(0.017)$}  & $\mathbf{1.010}$ {\scriptsize$(0.025)$}  \\
        \bottomrule
    \end{tabular}
    }
    \vspace{-5pt}
\end{table}

\subsection{Human Evaluation}

\begin{figure}[!htbp]
    \centering
    \begin{subfigure}[b]{0.49\linewidth}
        \centering
        \includegraphics[width=\linewidth]{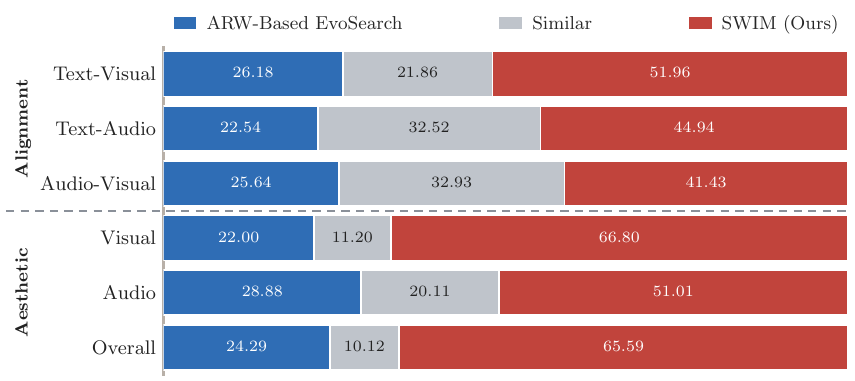}
        \caption{JavisDiT++}
        \label{fig:human-eval-javis}
    \end{subfigure}
    \hfill
    \begin{subfigure}[b]{0.49\linewidth}
        \centering
        \includegraphics[width=\linewidth]{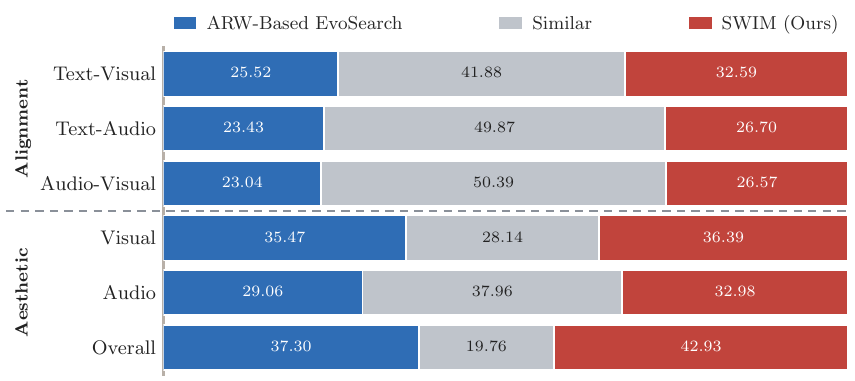}
        \caption{LTX-2}
        \label{fig:human-eval-ltx}
    \end{subfigure}
    \caption{Human preference responses comparing \modelname JavisDiT++ and LTX-2 on \textsc{JavisBench-Mini}~\citep{javis}. Bars show relative preferences on text-visual, text-audio, and audio-visual alignments, as well as visual, audio, and overall aesthetics. \modelname is preferred over the corresponding baseline across all the preference criteria.}
    \label{fig:human-eval}
    \vspace{-5pt}
\end{figure}

The reported human evaluation compares adaptive \modelname with ARW-based EvoSearch, with each video generated on a budget of 560 NFE. The study comprises 80 prompts from \textsc{JavisBench-Mini}~\citep{javis} (all available in the supplementary). We recruited 62 participants in total and showed each participant 25 random prompts from these 80 along with two videos, one using ARW-based EvoSearch and the other using \modelname, in random order. In each pair, the base model was either JavisDiT++ or LTX-2, chosen randomly and independently across pairs. Raters judged each pair on three alignment criteria, text-visual, text-audio, and audio-visual, and three aesthetic criteria, visual, audio, and overall. \Cref{fig:human-eval} reports the resulting preference shares, pooled over 741 pairs of JavisDiT++ videos and over 764 pairs of LTX-2 videos. \modelname outperforms both JavisDiT++ and LTX-2 baselines across all six preference criteria, with preference margins of $15-44\%$ over JavisDiT++ and $1-7\%$ over LTX-2.

\subsection{Ablations}\label{subsec:ablation}

\begin{figure}[t]
    \centering
    \begin{minipage}[t]{0.49\linewidth}
        \centering
        \includegraphics[width=0.85\linewidth]{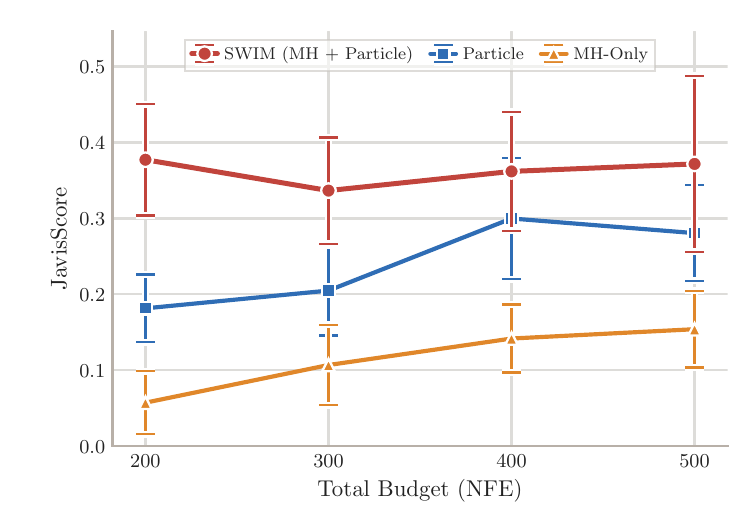}
        \caption{Particle-only, MH-only, and \modelname at matched total NFE. \modelname has the highest mean JavisScore at each tested budget, with its largest mean advantage over particle-only search at the lowest budget.}
        \label{fig:mh-refinement}
    \end{minipage}
    \hfill
    \begin{minipage}[t]{0.49\linewidth}
        \centering
        \includegraphics[width=0.85\linewidth]{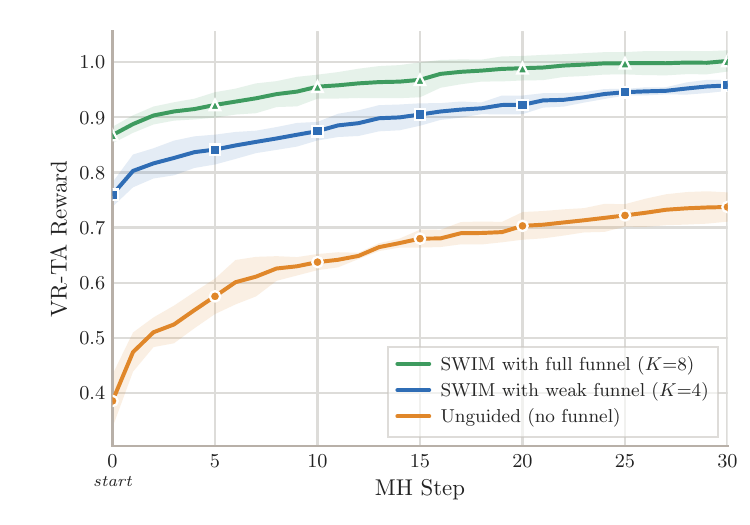}
        \caption{VR-TA reward over $T{=}30$ MH steps from an unguided sample, a weak funnel ($K{=}4$), and a full funnel ($K{=}8$), with $s_{\mathrm{mh}}{=}0.7$. Stronger starts reach higher final reward but leave less improvement for MH.}
        \label{fig:mh-initialization}
    \end{minipage}
    \vspace{-5pt}
\end{figure}

\begin{figure}[!htbp]
    \centering
    \begin{subfigure}[b]{0.49\linewidth}
        \centering
        \includegraphics[width=0.85\linewidth]{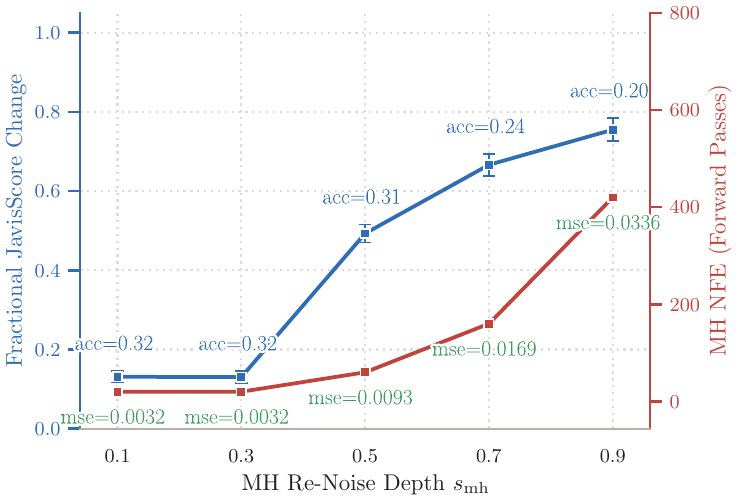}
        \caption{JavisDiT++}
        \label{fig:mh-smid-javis}
    \end{subfigure}
    \hfill
    \begin{subfigure}[b]{0.49\linewidth}
        \centering
        \includegraphics[width=0.85\linewidth]{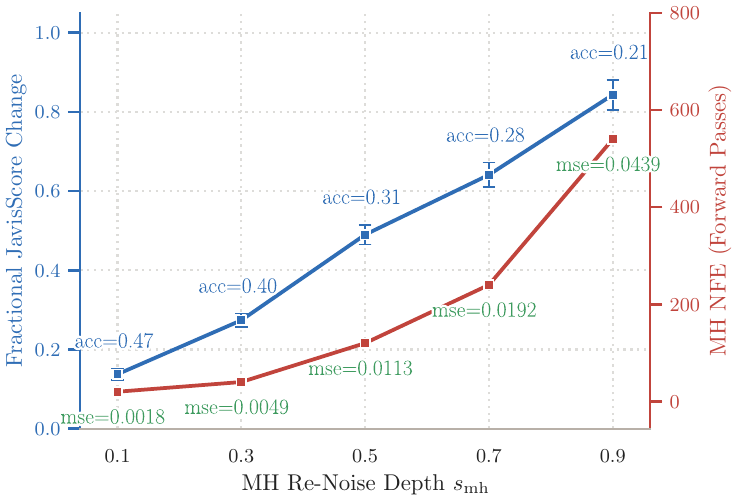}
        \caption{LTX-2}
        \label{fig:mh-smid-ltx}
    \end{subfigure}
    \caption{Effect of MH noise level $s_{\mathrm{mh}}$ on JavisDiT++ (left) and LTX-2 (right). Bars show relative JavisScore improvement over the funnel-only baseline and lines show MH NFE cost. The reported setting uses $s_{\mathrm{mh}}=0.7$.}
    \label{fig:mh-smid}
\end{figure}

\paragraph{Particle search versus MH refinement.}

\Cref{fig:mh-refinement} compares relative improvements in JavisScore for particle-only search, MH-only refinement, and \modelname at matched NFE. Results on JavisDiT++ are means over 50 prompts and 3 seeds. MH-only has the smallest relative improvement, as starting from an unguided candidate requires substantial compute for effective mixing. Particle-only search achieves larger improvements through global exploration but remains limited by finite-population estimation error, especially at low budgets. \modelname achieves the highest relative improvement by using particle search to find a promising candidate and MH to refine it. Its advantage over particle-only search is largest at low budgets and narrows as compute increases, consistent with MH correcting some estimation error left by a smaller particle population.

\paragraph{Quality of the MH initialization.}

\Cref{fig:mh-initialization} compares MH-only, a weak funnel ($K=4$) followed by MH, and a stronger funnel ($K=8$) followed by MH. Results on JavisDiT++ are means over 50 prompts and 3 seeds. All settings use the same VR-TA reward and MH schedule. MH-only improves initially but quickly plateaus below both funnel-initialized chains. The stronger funnel starts highest and peaks earlier. The weaker funnel starts lower but gains more during MH, closing much of the gap and reaching a comparable final reward with fewer particles. This suggests that a strong start accelerates reward gains, while MH can recover some of the advantage of a larger particle population.

\begin{figure}[!htbp]
    \centering
    \begin{subfigure}[b]{0.49\linewidth}
        \centering
        \includegraphics[width=0.85\linewidth]{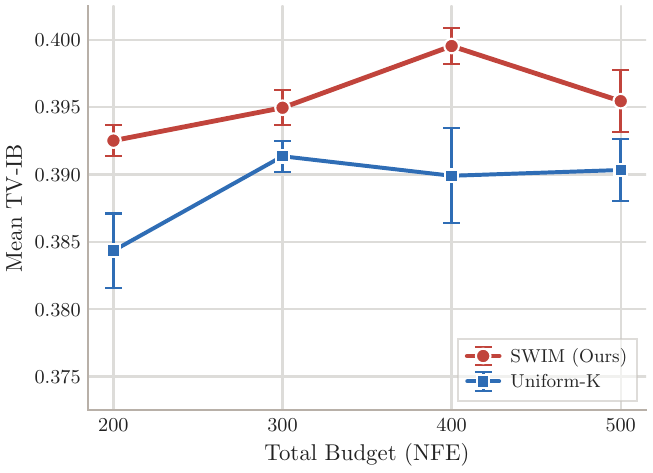}
        \caption{TV-IB vs. NFE}
        \label{fig:tvib-scaling-main-score}
    \end{subfigure}
    \hfill
    \begin{subfigure}[b]{0.49\linewidth}
        \centering
        \includegraphics[width=0.85\linewidth]{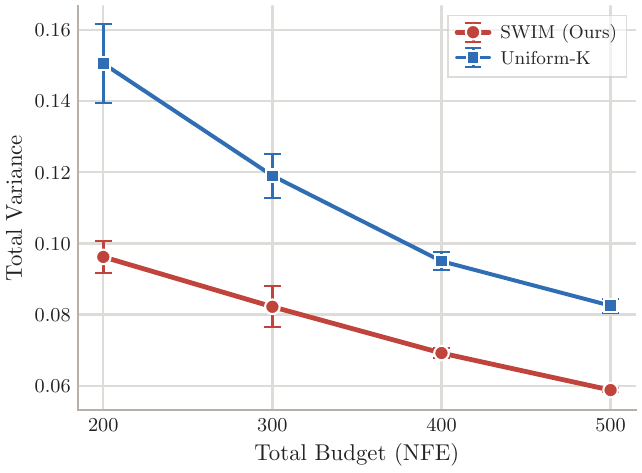}
        \caption{Variance vs. NFE}
        \label{fig:tvib-scaling-main-variance}
    \end{subfigure}
    \caption{TV-IB and estimator-variance surrogate versus NFE on JavisDiT++. The left panel shows TV-IB and the right panel shows the variance surrogate across compute budgets. Adaptive allocation achieves higher TV-IB and lower variance than uniform-$K$ throughout the tested range.}
    \label{fig:tvib-scaling-main}
\end{figure}

\paragraph{\textbf{Effect of $s_{\mathrm{mh}}$ in MH refinement.}}
In the MH phase, $s_{\mathrm{mh}}$ sets the noise level used to perturb the latent before re-denoising, and thus controls proposal size. \Cref{fig:mh-smid} sweeps $s_{\mathrm{mh}} \in \{0.1,0.3,0.5,0.7,0.9\}$ on JavisDiT++ and LTX-2, reporting JavisScore improvement over the funnel-only baseline and MH NFE cost. Results are means over 50 prompts and 3 seeds per model. At low $s_{\mathrm{mh}}$ ($\leq 0.3$), proposals are conservative and yield small improvements despite high acceptance. At high values ($\geq 0.8$), re-denoising from near-pure noise can discard the funnel's coarse alignment, making proposals resemble independent generations rather than local refinements. Around $s_{\mathrm{mh}}=0.7$, proposals improve audio-visual timing while preserving structure, balancing useful changes against destructive proposals. We use this setting in subsequent experiments.

\paragraph{Performance and variance across budgets.}
We use VR-TA to guide particle allocation for JavisDiT++ and report TV-IB as a separate evaluation metric to avoid tuning and evaluating on the same reward (\Cref{fig:tvib-scaling-main}). Results are means over 50 prompts and 3 seeds. Adaptive allocation yields higher TV-IB and lower estimator variance than uniform-$K$ across budgets. Score gains are smallest at the extremes: little compute can be redistributed at low budgets, while uniform allocation closes part of the gap at high budgets. The variance gap also narrows at the largest budget but remains visible. SD1.5 shows a similar pattern, with smaller quality and variance benefits as NFE increases (\Cref{fig:aes-scaling}).

\section{Limitations and Conclusion}\label{sec:conclusion}
We present \modelname as a two-stage method for efficient inference-time alignment. Stage 1 uses variance-optimal particle allocation to search for a candidate. Stage 2 uses Metropolis-Hastings (MH) to refine it locally, addressing reward-estimation error from finite particles. Online ESS guides allocation toward informative denoising stages. Together, particle search and MH combine global exploration with local refinement within a limited budget, without offline calibration or reward-specific tuning. The allocation theorem $K_m^* \propto \sigma_m/\sqrt{L_m}$ is model- and reward-agnostic and applies to diffusion and flow-matching models. At larger budgets, the smaller gains from adaptive allocation may not justify its extra schedule parameters. That being said, large models still incur substantial inference cost. Step distillation, latent caching, or token merging could reduce per-step cost alongside our budget-aware scheduling. The persistence forecast $\hat{\sigma}^2_{m+1|m} = \hat{\sigma}^2_m$ may be suboptimal when variance changes rapidly between adjacent checkpoints.

\bibliography{iclr2027_conference}

@misc{ldm,
      title={High-Resolution Image Synthesis with Latent Diffusion Models}, 
      author={Robin Rombach and Andreas Blattmann and Dominik Lorenz and Patrick Esser and Björn Ommer},
      year={2022},
      eprint={2112.10752},
      archivePrefix={arXiv},
      primaryClass={cs.CV},
      url={https://arxiv.org/abs/2112.10752}, 
}

@misc{lfm,
      title={Scaling Rectified Flow Transformers for High-Resolution Image Synthesis}, 
      author={Patrick Esser and Sumith Kulal and Andreas Blattmann and Rahim Entezari and Jonas Müller and Harry Saini and Yam Levi and Dominik Lorenz and Axel Sauer and Frederic Boesel and Dustin Podell and Tim Dockhorn and Zion English and Kyle Lacey and Alex Goodwin and Yannik Marek and Robin Rombach},
      year={2024},
      eprint={2403.03206},
      archivePrefix={arXiv},
      primaryClass={cs.CV},
      url={https://arxiv.org/abs/2403.03206}, 
}

@misc{svd,
      title={Stable Video Diffusion: Scaling Latent Video Diffusion Models to Large Datasets}, 
      author={Andreas Blattmann and Tim Dockhorn and Sumith Kulal and Daniel Mendelevitch and Maciej Kilian and Dominik Lorenz and Yam Levi and Zion English and Vikram Voleti and Adam Letts and Varun Jampani and Robin Rombach},
      year={2023},
      eprint={2311.15127},
      archivePrefix={arXiv},
      primaryClass={cs.CV},
      url={https://arxiv.org/abs/2311.15127}, 
}

@misc{wan,
      title={Wan: Open and Advanced Large-Scale Video Generative Models}, 
      author={Team Wan and Ang Wang and Baole Ai and Bin Wen and Chaojie Mao and Chen-Wei Xie and Di Chen and Feiwu Yu and Haiming Zhao and Jianxiao Yang and Jianyuan Zeng and Jiayu Wang and Jingfeng Zhang and Jingren Zhou and Jinkai Wang and Jixuan Chen and Kai Zhu and Kang Zhao and Keyu Yan and Lianghua Huang and Mengyang Feng and Ningyi Zhang and Pandeng Li and Pingyu Wu and Ruihang Chu and Ruili Feng and Shiwei Zhang and Siyang Sun and Tao Fang and Tianxing Wang and Tianyi Gui and Tingyu Weng and Tong Shen and Wei Lin and Wei Wang and Wei Wang and Wenmeng Zhou and Wente Wang and Wenting Shen and Wenyuan Yu and Xianzhong Shi and Xiaoming Huang and Xin Xu and Yan Kou and Yangyu Lv and Yifei Li and Yijing Liu and Yiming Wang and Yingya Zhang and Yitong Huang and Yong Li and You Wu and Yu Liu and Yulin Pan and Yun Zheng and Yuntao Hong and Yupeng Shi and Yutong Feng and Zeyinzi Jiang and Zhen Han and Zhi-Fan Wu and Ziyu Liu},
      year={2025},
      eprint={2503.20314},
      archivePrefix={arXiv},
      primaryClass={cs.CV},
      url={https://arxiv.org/abs/2503.20314}, 
}

@misc{aldm,
      title={AudioLDM: Text-to-Audio Generation with Latent Diffusion Models}, 
      author={Haohe Liu and Zehua Chen and Yi Yuan and Xinhao Mei and Xubo Liu and Danilo Mandic and Wenwu Wang and Mark D. Plumbley},
      year={2023},
      eprint={2301.12503},
      archivePrefix={arXiv},
      primaryClass={cs.SD},
      url={https://arxiv.org/abs/2301.12503}, 
}

@misc{fastaldm,
      title={Fast Timing-Conditioned Latent Audio Diffusion}, 
      author={Zach Evans and CJ Carr and Josiah Taylor and Scott H. Hawley and Jordi Pons},
      year={2024},
      eprint={2402.04825},
      archivePrefix={arXiv},
      primaryClass={cs.SD},
      url={https://arxiv.org/abs/2402.04825}, 
}

@inproceedings{javisdit,
      title={JavisDiT++: Unified Modeling and Optimization for Joint Audio-Video Generation}, 
      author={Kai Liu and Yanhao Zheng and Kai Wang and Shengqiong Wu and Rongjunchen Zhang and Jiebo Luo and Dimitrios Hatzinakos and Ziwei Liu and Hao Fei and Tat-Seng Chua},
      year={2026},
      booktitle={International Conference on Learning Representations},
      eprint={2602.19163},
      archivePrefix={arXiv},
      primaryClass={cs.CV},
      url={https://openreview.net/pdf?id=hRRWfFpKRp},
}

@misc{ltx2,
      title={LTX-2: Efficient Joint Audio-Visual Foundation Model}, 
      author={Yoav HaCohen and Benny Brazowski and Nisan Chiprut and Yaki Bitterman and Andrew Kvochko and Avishai Berkowitz and Daniel Shalem and Daphna Lifschitz and Dudu Moshe and Eitan Porat and Eitan Richardson and Guy Shiran and Itay Chachy and Jonathan Chetboun and Michael Finkelson and Michael Kupchick and Nir Zabari and Nitzan Guetta and Noa Kotler and Ofir Bibi and Ori Gordon and Poriya Panet and Roi Benita and Shahar Armon and Victor Kulikov and Yaron Inger and Yonatan Shiftan and Zeev Melumian and Zeev Farbman},
      year={2026},
      eprint={2601.03233},
      archivePrefix={arXiv},
      primaryClass={cs.CV},
      url={https://arxiv.org/abs/2601.03233}, 
}

@misc{ddpo,
      title={Training Diffusion Models with Reinforcement Learning}, 
      author={Kevin Black and Michael Janner and Yilun Du and Ilya Kostrikov and Sergey Levine},
      year={2024},
      eprint={2305.13301},
      archivePrefix={arXiv},
      primaryClass={cs.LG},
      url={https://arxiv.org/abs/2305.13301}, 
}

@misc{seedpo,
      title={SEE-DPO: Self Entropy Enhanced Direct Preference Optimization}, 
      author={Shivanshu Shekhar and Shreyas Singh and Tong Zhang},
      year={2025},
      eprint={2411.04712},
      archivePrefix={arXiv},
      primaryClass={cs.CV},
      url={https://arxiv.org/abs/2411.04712}, 
}

@misc{rocm,
      title={ROCM: RLHF on consistency models}, 
      author={Shivanshu Shekhar and Tong Zhang},
      year={2025},
      eprint={2503.06171},
      archivePrefix={arXiv},
      primaryClass={cs.LG},
      url={https://arxiv.org/abs/2503.06171}, 
}

@misc{draft,
      title={Directly Fine-Tuning Diffusion Models on Differentiable Rewards}, 
      author={Kevin Clark and Paul Vicol and Kevin Swersky and David J Fleet},
      year={2024},
      eprint={2309.17400},
      archivePrefix={arXiv},
      primaryClass={cs.CV},
      url={https://arxiv.org/abs/2309.17400}, 
}

@misc{fkd,
      title={A General Framework for Inference-time Scaling and Steering of Diffusion Models}, 
      author={Raghav Singhal and Zachary Horvitz and Ryan Teehan and Mengye Ren and Zhou Yu and Kathleen McKeown and Rajesh Ranganath},
      year={2025},
      eprint={2501.06848},
      archivePrefix={arXiv},
      primaryClass={cs.LG},
      url={https://arxiv.org/abs/2501.06848}, 
}

@misc{vasr,
      title={VASR: Variance-Aware Systematic Resampling for Reward-Guided Diffusion}, 
      author={Shivanshu Shekhar and Sagnik Mukherjee and Jia Yi Zhang and Tong Zhang},
      year={2026},
      eprint={2604.06779},
      archivePrefix={arXiv},
      primaryClass={cs.AI},
      url={https://arxiv.org/abs/2604.06779}, 
}

@misc{dts,
      title={Diffusion Tree Sampling: Scalable inference-time alignment of diffusion models}, 
      author={Vineet Jain and Kusha Sareen and Mohammad Pedramfar and Siamak Ravanbakhsh},
      year={2025},
      eprint={2506.20701},
      archivePrefix={arXiv},
      primaryClass={cs.LG},
      url={https://arxiv.org/abs/2506.20701}, 
}

@misc{syncdpo,
      title={SyncDPO: Enhancing Temporal Synchronization in Video-Audio Joint Generation via Preference Learning}, 
      author={Xin Cheng and Xihua Wang and Ying Ba and Yuyue Wang and Kaisi Guan and Yinbo Wang and Wenpu Li and Ruihua Song},
      year={2026},
      eprint={2605.12179},
      archivePrefix={arXiv},
      primaryClass={cs.CV},
      url={https://arxiv.org/abs/2605.12179}, 
}

@misc{spongebob,
      title={SpongeBob: Sync-Aware Harmonious Audio-Visual Generative Editing}, 
      author={Sen Liang and Cong Wang and Fengbin Guan and Zhentao Yu and Yiting Lu and Yuanzhi Wang and Yuan Zhou and Xin Li and Zhibo Chen},
      year={2026},
      eprint={2605.25193},
      archivePrefix={arXiv},
      primaryClass={cs.CV},
      url={https://arxiv.org/abs/2605.25193}, 
}

@misc{itsavgen,
      title={Inference-Time Scaling for Joint Audio-Video Generation}, 
      author={Jaemin Jung and Kyeongha Rho and Inkyu Shin and Joon Son Chung},
      year={2026},
      eprint={2606.03183},
      archivePrefix={arXiv},
      primaryClass={cs.MM},
      url={https://arxiv.org/abs/2606.03183}, 
}

@misc{fkc,
      title={Feynman-Kac Correctors in Diffusion: Annealing, Guidance, and Product of Experts}, 
      author={Marta Skreta and Tara Akhound-Sadegh and Viktor Ohanesian and Roberto Bondesan and Alán Aspuru-Guzik and Arnaud Doucet and Rob Brekelmans and Alexander Tong and Kirill Neklyudov},
      year={2025},
      eprint={2503.02819},
      archivePrefix={arXiv},
      primaryClass={cs.LG},
      url={https://arxiv.org/abs/2503.02819}, 
}

@misc{ddpm,
      title={Denoising Diffusion Probabilistic Models}, 
      author={Jonathan Ho and Ajay Jain and Pieter Abbeel},
      year={2020},
      eprint={2006.11239},
      archivePrefix={arXiv},
      primaryClass={cs.LG},
      url={https://arxiv.org/abs/2006.11239}, 
}

@misc{smld,
      title={Score-Based Generative Modeling through Stochastic Differential Equations}, 
      author={Yang Song and Jascha Sohl-Dickstein and Diederik P. Kingma and Abhishek Kumar and Stefano Ermon and Ben Poole},
      year={2021},
      eprint={2011.13456},
      archivePrefix={arXiv},
      primaryClass={cs.LG},
      url={https://arxiv.org/abs/2011.13456}, 
}

@misc{fm,
      title={Flow Matching for Generative Modeling}, 
      author={Yaron Lipman and Ricky T. Q. Chen and Heli Ben-Hamu and Maximilian Nickel and Matt Le},
      year={2023},
      eprint={2210.02747},
      archivePrefix={arXiv},
      primaryClass={cs.LG},
      url={https://arxiv.org/abs/2210.02747}, 
}

@misc{rfm,
      title={Flow Straight and Fast: Learning to Generate and Transfer Data with Rectified Flow}, 
      author={Xingchao Liu and Chengyue Gong and Qiang Liu},
      year={2022},
      eprint={2209.03003},
      archivePrefix={arXiv},
      primaryClass={cs.LG},
      url={https://arxiv.org/abs/2209.03003}, 
}

@misc{imagereward,
      title={ImageReward: Learning and Evaluating Human Preferences for Text-to-Image Generation}, 
      author={Jiazheng Xu and Xiao Liu and Yuchen Wu and Yuxuan Tong and Qinkai Li and Ming Ding and Jie Tang and Yuxiao Dong},
      year={2023},
      eprint={2304.05977},
      archivePrefix={arXiv},
      primaryClass={cs.CV},
      url={https://arxiv.org/abs/2304.05977}, 
}

@article{drawbench,
  title={Photorealistic text-to-image diffusion models with deep language understanding},
  author={Saharia, Chitwan and Chan, William and Saxena, Saurabh and Li, Lala and Whang, Jay and Denton, Emily L and Ghasemipour, Kamyar and Gontijo Lopes, Raphael and Karagol Ayan, Burcu and Salimans, Tim and others},
  journal={Advances in Neural Information Processing Systems},
  volume={35},
  pages={36479--36494},
  year={2022},
  url={https://papers.nips.cc/paper_files/paper/2022/hash/ec795aeadae0b7d230fa35cbaf04c041-Abstract-Conference.html}
}

@misc{aes,
  title = {Aesthetic Predictor},
  author = {{LAION-AI}},
  year = {2022},
  url = {https://github.com/LAION-AI/aesthetic-predictor}
}

@article{smc,
author = {Del Moral, Pierre and Doucet, Arnaud and Jasra, Ajay},
title = {Sequential Monte Carlo samplers},
journal = {Journal of the Royal Statistical Society: Series B (Statistical Methodology)},
volume = {68},
number = {3},
pages = {411-436},
doi = {10.1111/j.1467-9868.2006.00553.x},
url = {https://academic.oup.com/jrsssb/article-abstract/68/3/411/7110641},
eprint = {https://rss.onlinelibrary.wiley.com/doi/pdf/10.1111/j.1467-9868.2006.00553.x},
year = {2006}
}

@article{mh,
  author = {Metropolis, Nicholas and Rosenbluth, Arianna W. and Rosenbluth, Marshall N. and Teller, Augusta H. and Teller, Edward},
  title = {Equation of State Calculations by Fast Computing Machines},
  journal = {The Journal of Chemical Physics},
  volume = {21},
  number = {6},
  pages = {1087--1092},
  year = {1953},
  doi = {10.1063/1.1699114},
  url = {https://doi.org/10.1063/1.1699114}
}

@article{mh2,
 ISSN = {00063444, 14643510},
 URL = {http://www.jstor.org/stable/2334940},
 author = {W. K. Hastings},
 journal = {Biometrika},
 number = {1},
 pages = {97--109},
 publisher = {[Oxford University Press, Biometrika Trust]},
 title = {Monte Carlo Sampling Methods Using Markov Chains and Their Applications},
 urldate = {2026-09-05},
 volume = {57},
 doi = {10.1093/biomet/57.1.97},
 year = {1970}
}

@inproceedings{javis,
 author = {Liu, Kai and Li, Wei and Chen, Lai and Wu, Shengqiong and Zheng, Yanhao and Ji, Jiayi and Zhou, Fan and Luo, Jiebo and Liu, Ziwei and Fei, Hao (Scofield) and Chua, Tat-Seng},
 booktitle = {International Conference on Learning Representations},
 editor = {C. Vondrick and B. Hariharan and C. Raffel and L. Pinto and D. Yang and A. Faust},
 pages = {139160--139194},
 title = {JavisDiT: Joint Audio-Video Diffusion Transformer with Hierarchical Spatio-Temporal Prior Synchronization},
 url = {https://proceedings.iclr.cc/paper_files/paper/2026/file/e11afe29671baa67c9f15fa77fc97357-Paper-Conference.pdf},
 volume = {2026},
 year = {2026}
}

@misc{ib,
      title={ImageBind: One Embedding Space To Bind Them All}, 
      author={Rohit Girdhar and Alaaeldin El-Nouby and Zhuang Liu and Mannat Singh and Kalyan Vasudev Alwala and Armand Joulin and Ishan Misra},
      year={2023},
      eprint={2305.05665},
      archivePrefix={arXiv},
      primaryClass={cs.CV},
      url={https://arxiv.org/abs/2305.05665}, 
}

@inproceedings{avh,
author = {Mao, Yuxin and Shen, Xuyang and Zhang, Jing and Qin, Zhen and Zhou, Jinxing and Xiang, Mochu and Zhong, Yiran and Dai, Yuchao},
title = {TAVGBench: Benchmarking Text to Audible-Video Generation},
year = {2024},
isbn = {9798400706868},
publisher = {Association for Computing Machinery},
address = {New York, NY, USA},
url = {https://doi.org/10.1145/3664647.3680612},
doi = {10.1145/3664647.3680612},
booktitle = {Proceedings of the 32nd ACM International Conference on Multimedia},
pages = {6607–6616},
numpages = {10},
location = {Melbourne VIC, Australia},
series = {MM '24}
}

@misc{clap,
      title={Large-scale Contrastive Language-Audio Pretraining with Feature Fusion and Keyword-to-Caption Augmentation}, 
      author={Yusong Wu and Ke Chen and Tianyu Zhang and Yuchen Hui and Marianna Nezhurina and Taylor Berg-Kirkpatrick and Shlomo Dubnov},
      year={2024},
      eprint={2211.06687},
      archivePrefix={arXiv},
      primaryClass={cs.SD},
      url={https://arxiv.org/abs/2211.06687}, 
}

@inproceedings{vr,
 author = {Liu, Jie and Liu, Gongye and Liang, Jiajun and Yuan, Ziyang and Liu, Xiaokun and Zheng, Mingwu and Wu, Xiele and Wang, Qiulin and Xia, Menghan and Wang, Xintao and Liu, Xiaohong and Yang, Fei and Wan, Pengfei and ZHANG, Di and Gai, Kun and Yang, Yujiu and Ouyang, Wanli},
 booktitle = {Advances in Neural Information Processing Systems},
 doi = {10.52202/085713-2750},
 editor = {D. Belgrave and C. Zhang and H. Lin and R. Pascanu and P. Koniusz and M. Ghassemi and N. Chen},
 pages = {82155--82192},
 publisher = {Curran Associates, Inc.},
 title = {Improving Video Generation with Human Feedback},
 url = {https://proceedings.neurips.cc/paper_files/paper/2025/file/76227feb18ea0ee40bd15cf02c33e18e-Paper-Conference.pdf},
 volume = {38, Main Conference},
 year = {2025}
}

@article{MH_mixing,
 ISSN = {10505164},
 URL = {http://www.jstor.org/stable/2245134},
 author = {G. O. Roberts and A. Gelman and W. R. Gilks},
 journal = {The Annals of Applied Probability},
 number = {1},
 pages = {110--120},
 publisher = {Institute of Mathematical Statistics},
 title = {Weak Convergence and Optimal Scaling of Random Walk Metropolis Algorithms},
 urldate = {2026-09-15},
 volume = {7},
 year = {1997}
}

@misc{dmh,
      title={Effective Test-Time Scaling of Discrete Diffusion through Iterative Refinement}, 
      author={Sanghyun Lee and Sunwoo Kim and Seungryong Kim and Jongho Park and Dongmin Park},
      year={2025},
      eprint={2511.05562},
      archivePrefix={arXiv},
      primaryClass={cs.LG},
      url={https://arxiv.org/abs/2511.05562}, 
}

@misc{llada,
      title={Large Language Diffusion Models}, 
      author={Shen Nie and Fengqi Zhu and Zebin You and Xiaolu Zhang and Jingyang Ou and Jun Hu and Jun Zhou and Yankai Lin and Ji-Rong Wen and Chongxuan Li},
      year={2025},
      eprint={2502.09992},
      archivePrefix={arXiv},
      primaryClass={cs.CL},
      url={https://arxiv.org/abs/2502.09992}, 
}

@misc{mdm,
      title={Simple and Effective Masked Diffusion Language Models}, 
      author={Subham Sekhar Sahoo and Marianne Arriola and Yair Schiff and Aaron Gokaslan and Edgar Marroquin and Justin T Chiu and Alexander Rush and Volodymyr Kuleshov},
      year={2024},
      eprint={2406.07524},
      archivePrefix={arXiv},
      primaryClass={cs.CL},
      url={https://arxiv.org/abs/2406.07524}, 
}

@misc{evo,
      title={Scaling Image and Video Generation via Test-Time Evolutionary Search}, 
      author={Haoran He and Jiajun Liang and Xintao Wang and Pengfei Wan and Di Zhang and Kun Gai and Ling Pan},
      year={2025},
      eprint={2505.17618},
      archivePrefix={arXiv},
      primaryClass={cs.CV},
      url={https://arxiv.org/abs/2505.17618}, 
}
\bibliographystyle{iclr2027_conference}

\section*{AI Use Statement}
Generative AI tools assisted with manuscript framing and editing, identifying and checking literature references, reviewing mathematical claims and proofs, and providing feedback on methodology and interpretation of reported results. The authors are responsible for the final content, including AI-assisted text and claims.

\newpage

\appendix
\crefalias{section}{appendixsec}
\crefalias{subsection}{appendixsec}
\crefalias{subsubsection}{appendixsec}

\section{Proofs and Derivations}\label{app:proofs}

\subsection{Why the ESS-Based Objective Is Reasonable}\label{app:signal-noise}

The squared coefficient of variation of checkpoint weights is estimated by $\hat\sigma_m^2=K_{m-1}/\mathrm{ESS}_m-1$. If $K$ weights were independent draws from a fixed incremental-weight distribution, the relative variance of their sample mean would equal its population squared coefficient of variation divided by $K$. This elementary calculation motivates the $\sigma_m^2/K_{m-1}$ terms in the allocation surrogate. Resampled trajectories share ancestors, however, so the full SMC variance generally includes dependence and propagation effects~\citep{smc,vasr}. We do not identify the surrogate with the exact variance of the final output.

Weight dispersion need not indicate a useful reward signal. A noisy proxy can separate particles strongly while predicting final quality poorly. For a fixed population and a small reward temperature, exponentially reweighting by an intermediate reward $R_m$ changes the mean of a terminal reward $R_0$ to first order by $\lambda\,\mathrm{Cov}(R_0, R_m)$. This is a local approximation, not an exact finite-population resampling formula. It explains why the staged reward delays the text-audio term until its intermediate estimate becomes more useful. The online allocator uses weight dispersion because it is available at each checkpoint. The experiments assess whether that proxy helps final quality.

\subsection{Optimal Allocation}\label{app:discrete-alloc}

\begin{proof}[Proof of \Cref{thm:opt-alloc}]
With $K_0$ fixed, let $B_{\mathrm{seg}}=\tilde{B}-K_0n_1$ be the budget available to the decision variables. Introduce the Lagrange multiplier $\nu\in\mathbb{R}$ for the compute constraint. The Lagrangian is
\begin{equation}
    \mathcal{L} = \sum_{m=1}^{M-1}\frac{\sigma_{m+1}^2}{K_m} + \nu\!\left(\sum_{m=1}^{M-1}K_m L_m - B_{\mathrm{seg}}\right).
\end{equation}
Stationarity with respect to each $K_m$ gives
\begin{equation}
    -\frac{\sigma_{m+1}^2}{K_m^2}+\nu L_m=0
    \quad\implies\quad
    K_m^*=\frac{\sigma_{m+1}}{\sqrt{\nu L_m}}.
\end{equation}
Substituting into the budget constraint gives
\begin{equation}
    \frac{1}{\sqrt{\nu}}\sum_{m=1}^{M-1}\sigma_{m+1}\sqrt{L_m}
    =B_{\mathrm{seg}}.
\end{equation}
Thus $\sqrt{\nu}=\mathcal{Z}/B_{\mathrm{seg}}$, where $\mathcal{Z}=\sum_{m=1}^{M-1}\sigma_{m+1}\sqrt{L_m}$. Substitution gives
\begin{equation}
    K_m^*=\frac{B_{\mathrm{seg}}}{\mathcal{Z}}\frac{\sigma_{m+1}}{\sqrt{L_m}}.
\end{equation}
The objective is convex for $K_m>0$, so these stationary values minimize it under the budget constraint.
\end{proof}

\begin{remark}[Continuous limit]
As $M \to \infty$ and $L_m \to dt$, the problem becomes $\min_{K(\cdot)} \int_0^1 \sigma^2(t)/K(t)\,dt$ subject to $\int_0^1 K(t)\,dt = B$. The same pointwise Lagrangian argument gives $K^*(t) = B\,\sigma(t)/\|\sigma\|_{L^1}$.
\end{remark}

\subsection{Uniform Pacing as a Minimax Safeguard}\label{app:minimax-pace}

Consider $J$ remaining segments with lengths $L_j$ and allocations $K_j$, subject to $\sum_{j=1}^{J}L_jK_j \le B_{\mathrm{rem}}$. Suppose their unknown variance coefficients satisfy $0 \le \sigma_j^2 \le \bar{\sigma}^2$, and measure worst-stage risk by $W(K)=\max_j \sigma_j^2/K_j$. The uniform population allocation minimizes the worst-case risk.

\begin{proof}
For any allocation, the worst case is $\sup_{\sigma_j^2 \le \bar{\sigma}^2} W(K)=\bar{\sigma}^2/\min_j K_j$. Also,
\begin{equation}
    \min_j K_j \le \frac{\sum_j L_jK_j}{\sum_j L_j} \le \frac{B_{\mathrm{rem}}}{\sum_j L_j}.
\end{equation}
Thus the worst-case risk is at least $\bar{\sigma}^2\sum_j L_j/B_{\mathrm{rem}}$, attained by setting $K_j=B_{\mathrm{rem}}/\sum_j L_j$ for every segment. This is the continuous relaxation, and integer populations are handled by rounding and the affordability cap.
\end{proof}

\subsection{Budget Feasibility}\label{app:budget}

\begin{theorem}\label{thm:budget-feasibility}
Assume $B \ge \sum_{j=1}^{M} L_j$ (the budget suffices for at least one particle per segment). If at every checkpoint $K_m \le K_m^{\max} = \lfloor(B_m - \sum_{j>m} L_j)/L_m\rfloor$, then $B_{m+1} \ge \sum_{j>m} L_j$ for all $m$.
\end{theorem}

\begin{proof}
By induction. Base: $B_1 \ge \sum_{j=1}^M L_j$ by assumption. Step: $K_m \le K_m^{\max}$ gives
\begin{equation}
    L_m\,K_m \le B_m - \sum_{j>m} L_j,
\end{equation}
so $B_{m+1} = B_m - L_m K_m \ge \sum_{j>m} L_j$.
\end{proof}

\subsection{Reward Reliability Under Gaussian Assumptions}\label{app:gaussian-rho}

The following example shows how reward estimates can become more informative as denoising progresses. It does not determine the variance of particle-selection weights.

\begin{proposition}\label{prop:app-gaussian}
Let $x_0 \sim \mathcal{N}(0, \sigma_0^2 I_d)$, $x_t = (1{-}t)x_0 + t\varepsilon$ with $\varepsilon \sim \mathcal{N}(0, I_d)$, and $R(x) = a^\top x + c$, where $a\ne 0$. Define reward reliability as $\rho(t)=\operatorname{Corr}(R(\mathbb{E}[x_0\mid x_t]),R(x_0))$. Then
\begin{equation}
    \rho(t) = \frac{(1-t)\sigma_0}{\sqrt{(1-t)^2\sigma_0^2 + t^2}}.
\end{equation}
\end{proposition}

\begin{proof}
The posterior $x_0 \mid x_t$ is Gaussian with covariance $\Sigma_{0|t} = t^2\sigma_0^2/((1{-}t)^2\sigma_0^2 + t^2) \cdot I_d$. For the linear reward, $\mathrm{Var}[R(\hat{x}_0) - R(x_0)] = \|a\|^2 t^2\sigma_0^2/((1{-}t)^2\sigma_0^2 + t^2)$ and $\mathrm{Var}[R(x_0)] = \sigma_0^2\|a\|^2$. Then $\rho^2 = 1 - t^2/((1{-}t)^2\sigma_0^2 + t^2)$.
\end{proof}

Thus the denoised reward becomes more correlated with the terminal reward as $t$ decreases in this Gaussian example. Checkpoint weight dispersion also depends on the reward, selection temperature, propagation, and shared particle ancestry, so this calculation does not predict a universal variance profile. The online scheduler uses the observed ESS at each checkpoint instead of assuming one.

\subsection{MH Acceptance for Noising-Denoising Proposals}\label{app:mh-acceptance}

We derive the simplified acceptance ratio used in the MH refinement phase.

\paragraph{Setup.}
The ideal target is $\pi(x) \propto p_{\mathrm{ref}}(x)\,\exp(\lambda R(x))$. The proposal $q(x' \mid x)$ noises the entire latent to level $s_{\mathrm{mh}}$ and re-denoises:
\begin{enumerate}
    \item Sample $x_s \sim p(x_s \mid x)$ by running the forward process to noise level $s_{\mathrm{mh}}$.
    \item Sample $x' \sim p_{\mathrm{rev}}(x' \mid x_s)$ by running the reverse SDE from $s_{\mathrm{mh}}$ to $0$ with fresh Brownian motion.
\end{enumerate}
The marginal proposal density is $q(x' \mid x) = \int p(x_s \mid x)\,p_{\mathrm{rev}}(x' \mid x_s)\,dx_s$.

\paragraph{Detailed balance of the bridge.}
If the reverse dynamics give the exact Bayes posterior of the forward noising process under $p_{\mathrm{ref}}$, then $p_{\mathrm{rev}}(\cdot \mid x_s) = p_{\mathrm{ref}}(\cdot \mid x_s)$. The joint density over the triplet $(x, x_s, x')$ is
\begin{equation}
    p(x, x_s, x') = p_{\mathrm{ref}}(x)\,p(x_s \mid x)\,p_{\mathrm{ref}}(x' \mid x_s).
\end{equation}
By Bayes' rule, $p_{\mathrm{ref}}(x)\,p(x_s \mid x) = p_s(x_s)\,p_{\mathrm{ref}}(x \mid x_s)$. Substituting:
\begin{equation}
    p(x, x_s, x') = p_s(x_s)\,p_{\mathrm{ref}}(x \mid x_s)\,p_{\mathrm{ref}}(x' \mid x_s).
\end{equation}
This is symmetric in $x$ and $x'$: both are conditionally independent given $x_s$, drawn from the same posterior. Marginalizing over $x_s$:
\begin{equation}\label{eqn:app-bridge-db}
    p_{\mathrm{ref}}(x)\,q(x' \mid x) = p_{\mathrm{ref}}(x')\,q(x \mid x').
\end{equation}

\paragraph{Cancellation.}
The MH acceptance is
\begin{equation}
\begin{aligned}
    \alpha
    &= \min\!\left\{1,\;\frac{\pi(x')\,q(x \mid x')}{\pi(x)\,q(x' \mid x)}\right\} \\
    &= \min\!\left\{1,\;\frac{p_{\mathrm{ref}}(x')e^{\lambda R(x')}}{p_{\mathrm{ref}}(x)e^{\lambda R(x)}} \cdot \frac{p_{\mathrm{ref}}(x)}{p_{\mathrm{ref}}(x')}\right\} \\
    &= \min\!\big\{1,\;e^{\lambda(R(x')-R(x))}\big\},
\end{aligned}
\end{equation}
where the second equality uses \Cref{eqn:app-bridge-db}. The $p_{\mathrm{ref}}$ factors cancel exactly.

\paragraph{Approximate case.}
With a learned model, $p_{\mathrm{rev}}(\cdot \mid x_s)$ only approximates $p_{\mathrm{ref}}(\cdot \mid x_s)$, so \Cref{eqn:app-bridge-db} need not hold exactly. The simplified reward-only acceptance ratio is therefore approximate, and the resulting chain is not guaranteed to have the exact reward-tilted reference distribution as its stationary distribution.

\subsection{Ideal Common-Target MH Refinement}\label{app:mh-convergence}

Let $\pi$ be a common target for both stages and $P$ an exact MH kernel that leaves $\pi$ invariant. Let $\mu_B$ be the unconditional distribution of the MH starting state after a particle budget $B$. This distribution includes the randomness of generation, resampling, and terminal selection. For any $T\geq0$, Markov kernels contract total variation, so
\begin{equation}\label{eqn:mh-initialization-bound}
    \|\mu_B P^T-\pi\|_{\mathrm{TV}}
    =\|(\mu_B-\pi)P^T\|_{\mathrm{TV}}
    \leq\|\mu_B-\pi\|_{\mathrm{TV}}.
\end{equation}
If $P$ additionally contracts zero-mass signed measures by a factor $\rho<1$, the right-hand side is at most $\rho^T\|\mu_B-\pi\|_{\mathrm{TV}}$. An invariant, irreducible, and aperiodic kernel converges under its usual regularity conditions, but its rate need not be useful at a short inference budget. \Cref{eqn:mh-initialization-bound} explains why a closer initialization can help under a shared target, but it does not prove that selecting the highest-reward particle makes $\mu_B$ closer to $\pi$.

For a proposal $q_s$ with acceptance probability $\alpha_s$, stationary acceptance is $A_\pi(s)=\mathbb{E}_{X\sim\pi,Y\sim q_s(\cdot\mid X)}[\alpha_s(X,Y)]$. It depends on the target and proposal, not the chain's initialization. A larger proposal can travel farther when accepted yet also have lower acceptance. Therefore, neither acceptance alone nor the jump radius alone determines finite-budget mixing. In our practical method, the particle and MH rewards differ and the learned reverse proposal may not preserve $\pi$. The bound above is therefore a design motivation, while the reported output scores characterize the implemented procedure.

\section{Additional Experiments}

\subsection{Staged Reward Ablations}

\begin{figure}[t]
    \centering
    \begin{subfigure}[b]{0.49\linewidth}
        \centering
        \includegraphics[width=\linewidth]{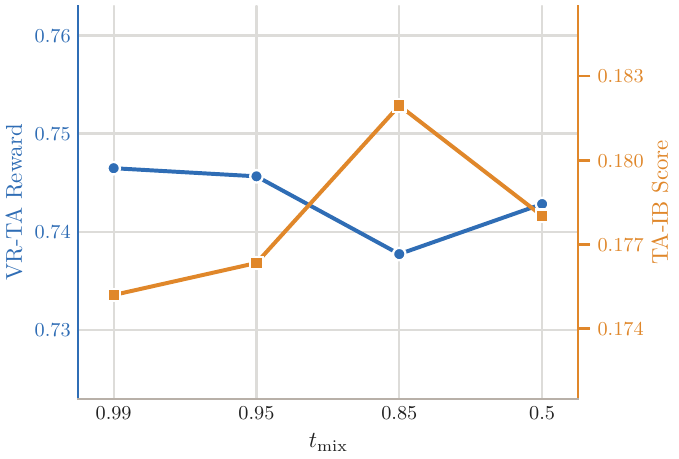}
        \caption{Effect of \textit{when} the text-audio reward $t_{\mathrm{mix}}$ is enabled.}
        \label{fig:mixing-start-ablation}
    \end{subfigure}
    \hfill
    \begin{subfigure}[b]{0.49\linewidth}
        \centering
        \includegraphics[width=\linewidth]{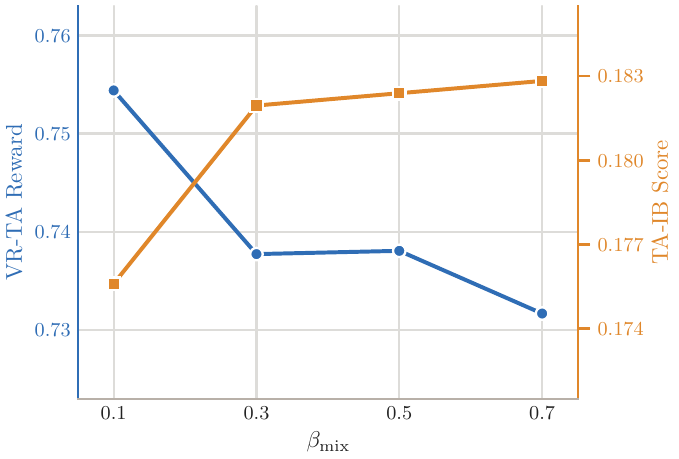}
        \caption{Effect of $\beta_{\mathrm{mix}}$ \textit{after} $t_{\mathrm{mix}}$ is enabled.}
        \label{fig:mixing-ratio-ablation}
    \end{subfigure}
    \caption{Effect of the staged reward parameters. The left panel shows the effect of when the text-audio reward is enabled, and the right panel shows its weight $\beta_{\mathrm{mix}}$. An intermediate activation threshold and moderate weight give the best balance between TA-IB and VR-TA.}
    \label{fig:staged-reward-ablation}
\end{figure}
The staged reward in \Cref{subsec:mh} introduces text-audio alignment at a selected denoising threshold $t_{\mathrm{mix}}$, with its strength controlled by $\beta_{\mathrm{mix}}$. \Cref{fig:staged-reward-ablation} examines the effect of both parameters. Introducing the text-audio reward too early ($t_{\mathrm{mix}}=0.95$ or $0.99$) leads to unreliable reward estimates at high noise levels and poor particle selection, consistent with the weak early-checkpoint discrimination of TA-IB in \Cref{fig:reward_reliability}, and degrades final alignment. Conversely, introducing it too late ($t_{\mathrm{mix}}=0.50$) leaves insufficient search time for the text-audio reward to influence the final output. We find $t_{\mathrm{mix}}=0.85$ provides the best balance between reliable reward estimates and sufficient search time. For reward strength, a small weight ($\beta_{\mathrm{mix}}=0.1$) underweights text-audio alignment, resulting in lower TA-IB despite higher VR-TA. Increasing it to $\beta_{\mathrm{mix}}=0.3$ substantially improves TA-IB with only a modest reduction in VR-TA, while larger weights ($0.5$ or $0.7$) offer diminishing TA-IB gains at the expense of VR-TA. Although performance differences across these settings are relatively modest, the trends favor $t_{\mathrm{mix}}=0.85$ and $\beta_{\mathrm{mix}}=0.3$, which we use in all subsequent experiments.

\subsection{Pacing Floor $\gamma$}

\begin{figure}
    \centering
    \begin{subfigure}[b]{0.49\linewidth}
        \centering
        \includegraphics[width=\linewidth]{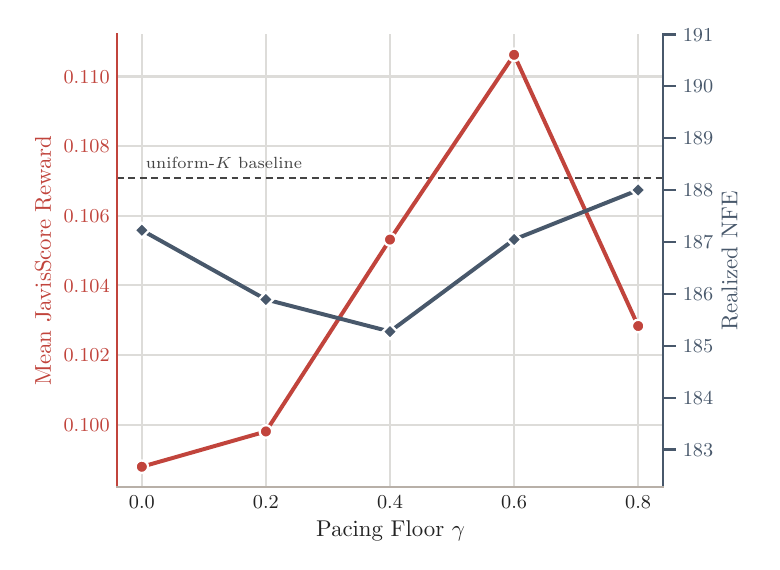}
        \caption{Final reward and realized NFE.}
        \label{fig:gamma_ablation_reward}
    \end{subfigure}
    \hfill
    \begin{subfigure}[b]{0.49\linewidth}
        \centering
        \includegraphics[width=\linewidth]{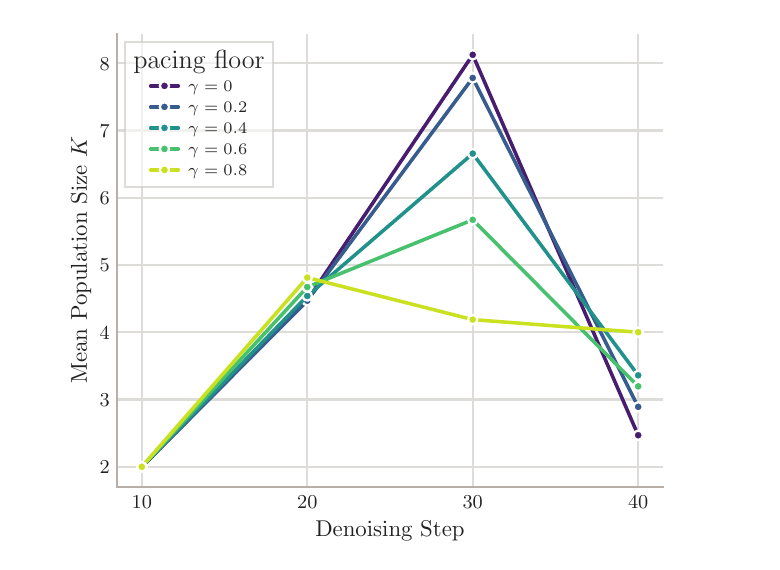}
        \caption{Particle schedule across checkpoints.}
        \label{fig:gamma_ablation_schedule}
    \end{subfigure}
    \caption{Effect of the pacing floor $\gamma$ on the adaptive funnel. Left: final reward and realized NFE. Right: actual particle schedule across checkpoints. Larger $\gamma$ forces the schedule toward a more uniform allocation, and intermediate pacing achieves the best reward.}
    \label{fig:gamma_ablation}
\end{figure}

The online allocator uses a persistence forecast for the next checkpoint's variance. This forecast is useful when neighboring checkpoints have similar reward informativeness, but it can overreact to a single high-variance stage and spend too much budget before later checkpoints. The pacing floor $K_{\mathrm{pace}}=\lceil \gamma B_{\mathrm{rem}}/\sum_{j\ge m}L_j\rceil$ prevents this failure mode by reserving a minimum particle population for future stages. \Cref{fig:gamma_ablation} shows the expected tradeoff. With small $\gamma$, the schedule is highly adaptive and places many particles at the informative middle checkpoint, but it can prune too aggressively afterward. With large $\gamma$, the schedule becomes closer to uniform and loses some of the benefit of adaptation. The best reward occurs at $\gamma=0.6$, indicating that the optimal behavior lies between a fully adaptive persistence forecast and a uniform allocation.

\subsection{Effect of Dual Step Size $\eta_\nu$}

\begin{figure}
    \centering
    \begin{subfigure}[b]{0.49\linewidth}
        \centering
        \includegraphics[width=\linewidth]{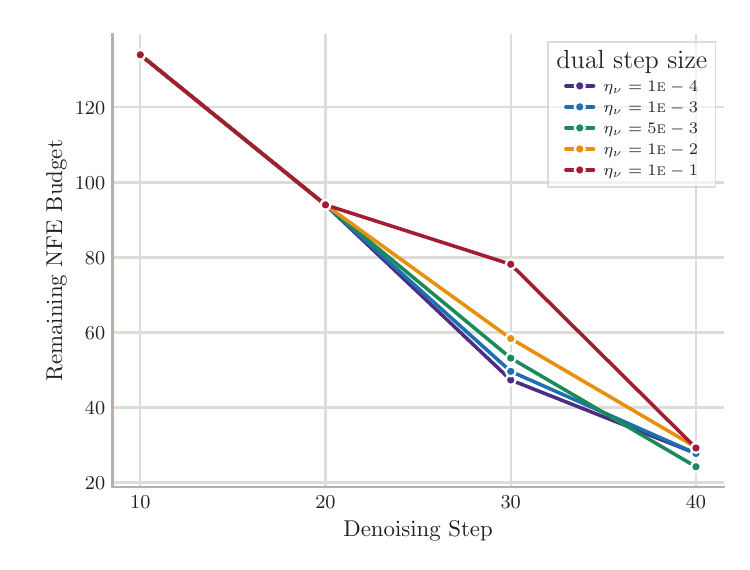}
        \caption{Remaining NFE budget.}
        \label{fig:nu_ablation_budget}
    \end{subfigure}
    \hfill
    \begin{subfigure}[b]{0.49\linewidth}
        \centering
        \includegraphics[width=\linewidth]{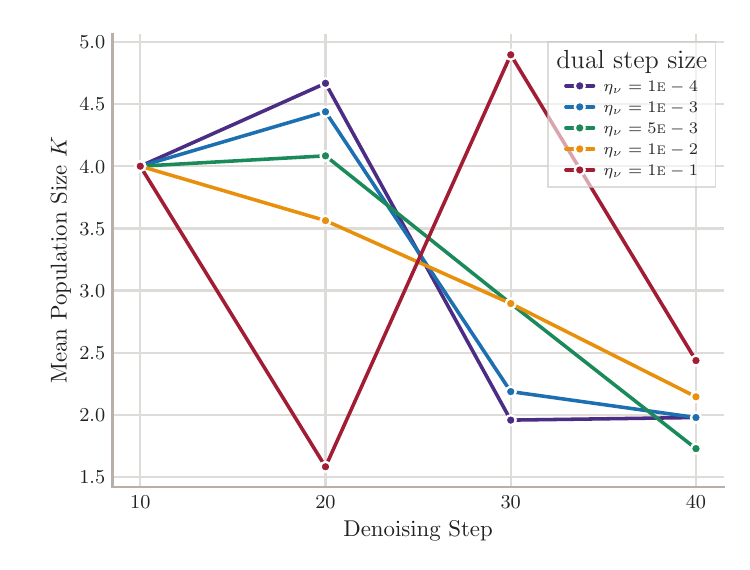}
        \caption{Mean population size $K$.}
        \label{fig:nu_ablation_k}
    \end{subfigure}
    \caption{Effect of the dual step size $\eta_\nu$ on the funnel's compute schedule under JavisDiT++. The left panel shows remaining NFE budget and the right panel shows mean population size $K$ across checkpoints. Larger $\eta_\nu$ responds more aggressively to budget errors, while schedules remain identical until the first update to $\nu$.}

    \label{fig:nu_ablation}
\end{figure}

The dual variable $\nu$ acts as the shadow price of compute: a larger $\nu$ induces smaller populations, while $\eta_\nu$ controls how quickly it adapts to budget over- or under-spending. Setting $\eta_\nu=0$ freezes $\nu$ at initialization, whereas larger values enable faster adaptation. \Cref{fig:nu_ablation} sweeps $\eta_\nu \in \{1\textsc{e}-4, 1\textsc{e}-3, 5\textsc{e}-3, 1\textsc{e}-2, 1\textsc{e}-1\}$ on JavisDiT, with all other parameters fixed. Because $\nu$ is initialized from the near-zero variance at the first checkpoint and the dual update only acts thereafter, the schedule is identical across $\eta_\nu$ through the first two checkpoints and diverges only later. From there, larger $\eta_\nu$ prunes more aggressively, as it is more responsive to the previous stage and rapidly increases or decreases $K$, while small $\eta_\nu$ keeps $K$ near $K_0$ and decreases it gradually. Total NFE stays near the cap for all values. Final performance follows a non-monotonic trend, with the highest reward at $\eta_\nu=10^{-3}$ ($0.215$) and the lowest at $\eta_\nu=10^{-1}$ ($0.145$). Moderate adaptation preserves particles until later denoising stages, where reward estimates are more informative, while overly aggressive adaptation causes frequent pruning and population growth, degrading performance.

\subsection{Variance and Performance Curves for Stable Diffusion}

\begin{figure}[t]
    \centering

    \begin{subfigure}[b]{0.49\linewidth}
        \centering
        \includegraphics[width=\linewidth]{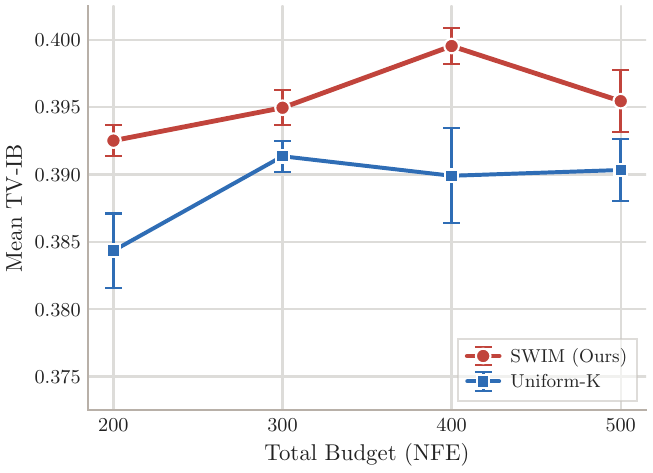}
        \caption{TV-IB vs.\ NFE.}
        \label{fig:tvib-scaling-score}
    \end{subfigure}
    \hfill
    \begin{subfigure}[b]{0.49\linewidth}
        \centering
        \includegraphics[width=\linewidth]{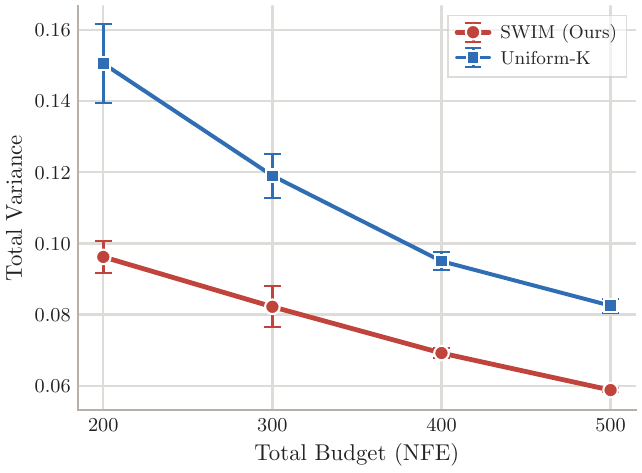}
        \caption{Estimator variance vs.\ NFE.}
        \label{fig:tvib-scaling-variance}
    \end{subfigure}

    \caption{Score and estimator variance under compute scaling for TV-IB in audio-visually aligned video generation. At full budget (NFE$\approx$500), both methods perform comparably on score but \modelname retains a substantial advantage on variance. As budget decreases, \modelname maintains higher quality with estimator variance increasing comparatively smoothly, while uniform-$K$ degrades more sharply. This behavior results from concentrating pruning at low-value checkpoints while preserving particles at informative middle checkpoints. Results on JavisDiT++ with 50 prompts and 3 seeds. Error bars show $\pm1$ standard error across prompts.}
    \label{fig:tvib-scaling}
\end{figure}

\begin{figure}[t]
    \centering

    \begin{subfigure}[b]{0.49\linewidth}
        \centering
        \includegraphics[width=\linewidth]{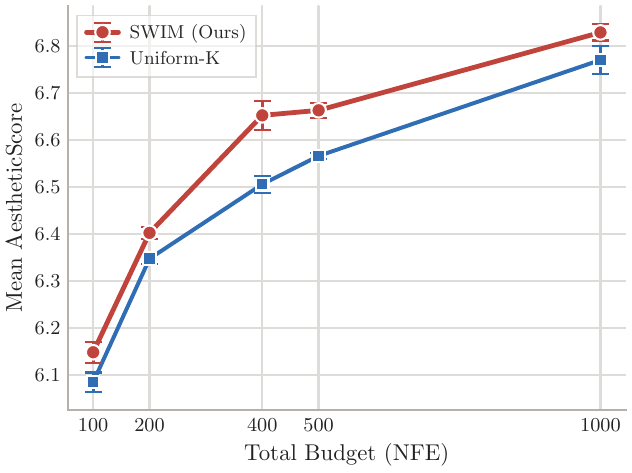}
        \caption{AestheticScore vs.\ NFE.}
        \label{fig:aes-scaling-score}
    \end{subfigure}
    \hfill
    \begin{subfigure}[b]{0.49\linewidth}
        \centering
        \includegraphics[width=\linewidth]{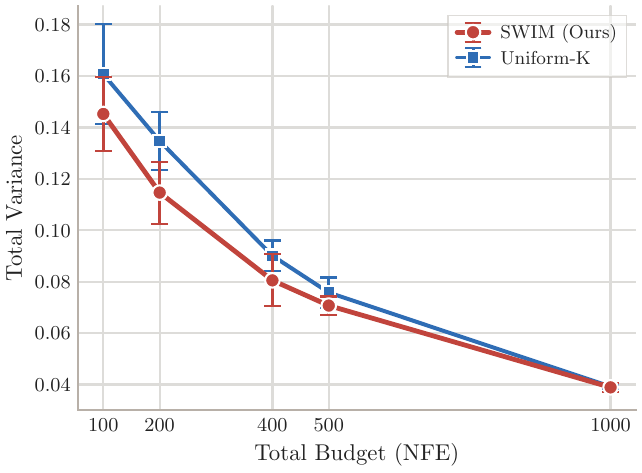}
        \caption{Estimator variance vs.\ NFE.}
        \label{fig:aes-scaling-variance}
    \end{subfigure}

    \caption{Score and estimator variance under compute scaling for AestheticScore in image generation. At full budget (NFE$\approx$1000), both methods perform comparably. As budget decreases, \modelname maintains higher quality with estimator variance increasing comparatively smoothly, while uniform-$K$ degrades more sharply. This behavior results from concentrating pruning at low-value checkpoints while preserving particles at informative middle checkpoints. Results on SD1.5 with 45 prompts and 3 seeds. Error bars show $\pm1$ standard error across prompts.}
    \label{fig:aes-scaling}
\end{figure}

\paragraph{Effect of compute scaling.}
Against corresponding uniform-$K$ baselines, \Cref{fig:tvib-scaling-score} compares the quality-compute tradeoff of \modelname on JavisDiT++ using TV-IB across NFE budgets from 200 to 500, and \Cref{fig:aes-scaling-score} compares the quality-compute tradeoff of \modelname on SD1.5 using AestheticScore across NFE budgets from 100 to 1000. \modelname consistently outperforms uniform-$K$ across all budgets. On TV-IB, the largest improvement of $2.5\%$ is observed at NFE=$400$, and on AestheticScore, the largest improvement of $2.3\%$ is also observed at NFE=$400$. As compute increases, the score gap narrows to $1.3\%$ at NFE=$500$ for TV-IB and to $0.9\%$ at NFE=$1000$ for AestheticScore. This trend is expected, as larger particle populations allow even uniform allocation to explore the reward landscape more effectively, reducing the benefit of adaptive allocation. The gains at low budgets are particularly relevant for expensive AV models, where the number of model evaluations is a major constraint.

\paragraph{Variance comparison.}
\Cref{fig:tvib-scaling-variance} and \Cref{fig:aes-scaling-variance} compare the respective total estimator variances $\sum_m \hat{\sigma}_m^2/K_m$ for TV-IB and AestheticScore under the same compute budgets. \modelname achieves lower variance across all budgets. At budgets of NFE=$200$ and NFE=$400$, the respective reductions are $36.1\%$ and $27.2\%$ for TV-IB and $14.9\%$ and $10.6\%$ for AestheticScore. As compute increases, the variance gap remains at $28.7\%$ at NFE=$500$ for TV-IB, while it narrows to $0.8\%$ at NFE=$1000$ for AestheticScore. The trend for AestheticScore mirrors the score improvements in \Cref{fig:aes-scaling-score}, showing that adaptive allocation provides the largest variance reductions and quality gains when particles are limited, while both methods converge as the population grows. The trend for TV-IB, on the other hand, indicates substantial benefits even at higher compute, which is a reflection of the significantly higher search space for audio-visual generation compared to images. Overall, these results support our variance-minimization objective and are consistent with the findings of VASR~\citep{vasr}, which empirically links lower estimator variance to better sample quality.

\subsection{Effect of Population Growth Strategy}

\begin{figure}[!htbp]
    \centering
    \includegraphics[width=0.75\linewidth]{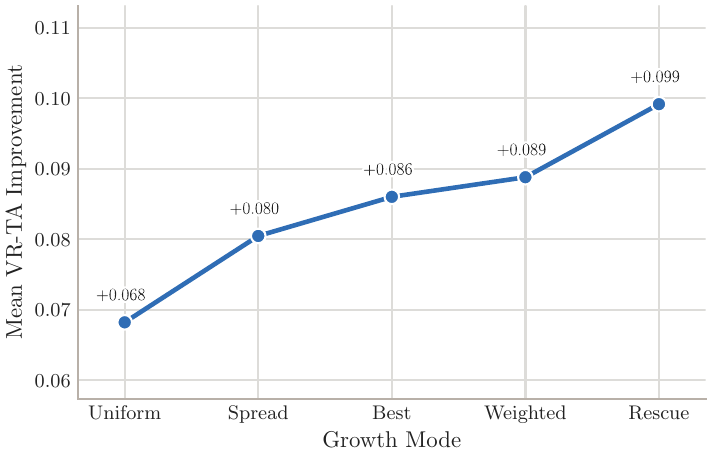}
    \caption{Effect of population growth strategy on VR-TA improvement over base generation. Rescue achieves the largest gain among five strategies, combining reward-guided resampling with lineage preservation during population expansion.Results on JavisDiT++ with 50 prompts and 3 seeds.}
    \label{fig:grow-mode}
\end{figure}

When the adaptive scheduler increases the particle population, new particles are generated by duplicating existing ones. We compare five growth strategies: \emph{best}, which clones the highest-scoring particle, \emph{weighted}, which samples according to importance weights, \emph{uniform}, which duplicates particles randomly, \emph{spread}, which prioritizes distinct lineages, and \emph{rescue}, which first resamples the population and then expands it by prioritizing particles with fewer offspring. \Cref{fig:grow-mode} compares their improvement in VR-TA over base generation. \emph{Rescue} achieves the largest improvement ($0.099$), followed by \emph{weighted} ($0.089$) and \emph{best} ($0.086$), while \emph{uniform} performs the worst ($0.068$). The population schedules remain similar across strategies, suggesting that the performance differences arise primarily from how new particles are selected. These results highlight the benefit of balancing reward-guided selection with lineage preservation during population growth. We therefore use \emph{rescue} in all subsequent experiments.

\begin{figure}[!htbp]
    \centering
    \begin{subfigure}[b]{0.49\linewidth}
        \centering
        \includegraphics[width=\linewidth]{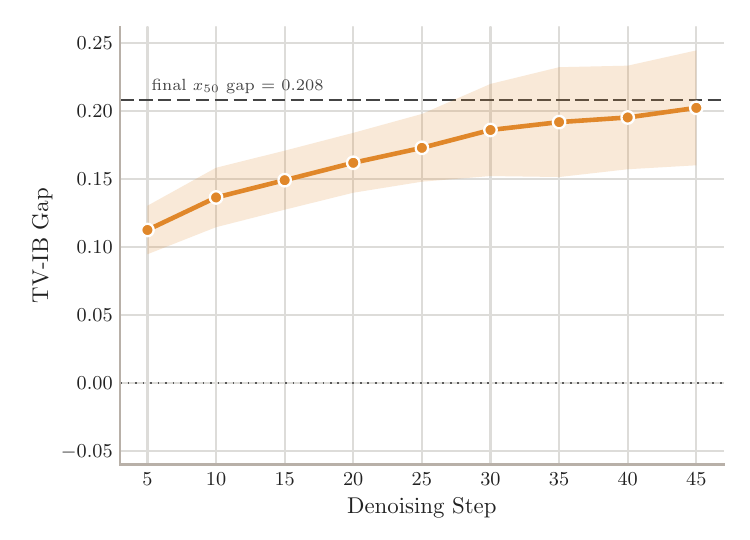}
        \caption{TV-IB.}
        \label{fig:reward_reliability_tvib}
    \end{subfigure}
    \hfill
    \begin{subfigure}[b]{0.49\linewidth}
        \centering
        \includegraphics[width=\linewidth]{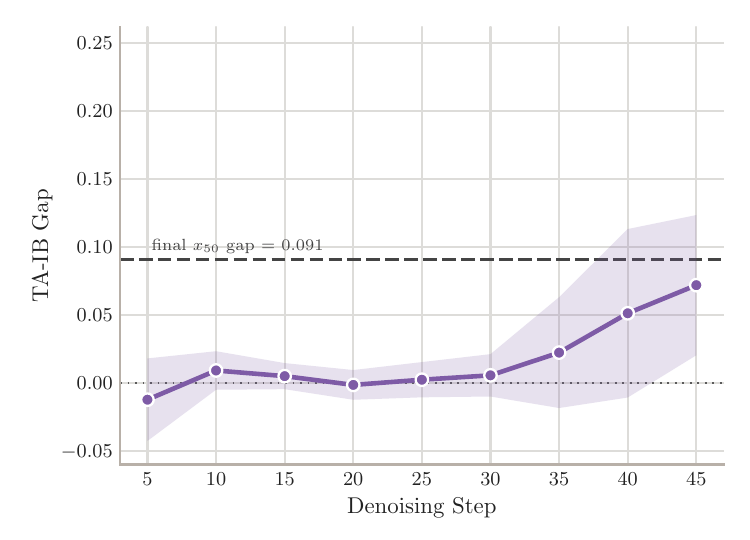}
        \caption{TA-IB.}
        \label{fig:reward_reliability_taib}
    \end{subfigure}
    \caption{Reward reliability across denoising checkpoints. We plot the score difference between two dissimilar prompts. Shaded regions show standard deviation across prompts and seeds. TA-IB has little discriminative power early, while TV-IB distinguishes prompts earlier. Both become more discriminative as denoising progresses.}
\label{fig:reward_reliability}
\end{figure}

\end{document}